\documentclass[10pt]{article}
\usepackage[margin=1.1in]{geometry} 

\usepackage[colorlinks,linkcolor=blue]{hyperref}

\usepackage{biblatex}
\usepackage{amssymb}
\usepackage{amsmath}
\usepackage{amsthm}
\usepackage{bbm}
\usepackage{multicol}
\usepackage[normalem]{ulem}
\usepackage{booktabs}
\usepackage{graphicx, subcaption}

\usepackage[colorinlistoftodos, textwidth=4cm, shadow]{todonotes}

\usepackage{thmtools}
\usepackage{verbatim}

\bibliography{references}

\renewcommand{\S}{\mathcal{S}}
\newcommand{\A}{\mathcal{A}}

\newcommand{\T}{\mathcal{T}}
\newcommand{\R}{\mathbb{R}}
\renewcommand{\P}{\mathbb{P}}
\newcommand{\E}{\mathbb{E}}
\newcommand{\U}{\mathbb{U}}
\newcommand{\M}{\mathbb{M}}
\newcommand{\NN}{\mathbb{N}}

\newcommand{\OCEVI}{OCE-VI}
\newcommand{\OCE}{\mathrm{OCE}}
\newcommand{\effh}{H_\gamma}

\newcommand{\argmax}{\arg\!\max}

\newcommand{\cA}{\mathcal{A}}

\newcommand{\cS}{\mathcal{S}}

\newcommand{\cX}{\mathcal{X}}

\newtheorem{theorem}{Theorem}

\newtheorem{definition}{Definition}
\declaretheorem[name=Lemma]{lemma}
\newtheorem{proposition}{Proposition}
\newtheorem{corollary}{Corollary}
\newtheorem{remark}{Remark}

\usepackage[linesnumbered,ruled,vlined]{algorithm2e}
\usepackage{mathtools}

\usepackage[toc,page,header]{appendix}
\usepackage{minitoc}

\title{On the Sample Complexity of Discounted Reinforcement Learning with Optimized Certainty Equivalents}

\author{Oliver Mortensen\footnote{Department of Computer Science, University of Copenhagen. Email: {\tt\small olmo@di.ku.dk}.}
        \and Mohammad Sadegh Talebi\footnote{Department of Computer Science, University of Copenhagen. Email: {\tt\small sadegh.talebi@di.ku.dk}.}}
\begin{document}

\maketitle  

\begin{abstract}

We study risk-sensitive reinforcement learning in finite discounted MDPs, where 
a generative model of the MDP is assumed to be available. We consider a family or risk measures called the optimized certainty equivalent (OCE), which includes important risk measures such as entropic risk, CVaR, and mean-variance. Our focus is on the sample complexities of learning the optimal state–action value function (value learning) and an optimal policy (policy learning) under recursive OCE. 
We provide an exact characterization of utility functions $u$ for which the corresponding OCE defines an objective that is PAC-learnable. We analyze a simple model-based approach and derive PAC sample complexity bounds.  
We establish that whenever $u$ does not have full domain $\text{dom}(u)\neq \mathbb{R}$, the corresponding problem is not PAC learnable. 
Finally, we establish corresponding lower bounds for both value and policy learning, demonstrating tightness in the size $SA$ of state-action space, and for a more restricted class of utilities, we derive lower bounds that makes the dependence on  the effective horizon $\frac{1}{1-\gamma}$ explicit. Specifically for $\text{CVaR}_\tau$ we show that the correct dependence on $\tau^{-1}$ is $\frac{1}{\tau^2}$ improving by a factor of $\frac{1}{\tau}$ over state-of-the-art although our bound has a suboptimal dependence on $\frac{1}{1-\gamma}$.

\end{abstract}

\section{Introduction}

In reinforcement learning (RL), the standard objective is to maximize the expected return, defined as the (possibly discounted) sum of rewards \cite{sutton1998reinforcement}. 
However, because this objective is inherently \emph{risk-neutral}, it may be inadequate for many high-stakes application domains, such as treatment \cite{ernst2006clinical}, finance \cite{scutella2013robust,bielecki1999risk}, operations research \cite{delage2010percentile}, and transportation \cite{kamran2020risk}. Such applications demand for account the variability of returns, and risks thereof.  One principled approach to addressing this limitation is to optimize a risk measure of the return distribution. 
which using concave risk measures leads to well-defined optimization problems. Notable risk measures include mean-variance \cite{li2000optimal}, value-at-risk (VaR) \cite{dempster2002risk}, Conditional VaR (CVaR) \cite{shapiro2021LNs}, entropic risk  \cite{howard1972risk}, and entropic VaR (EVaR) \cite{ahmadi2012entropic}, all of which have been applied to a wide-range of scenarios. Among these, CVaR has become particularly popular for modeling risk-sensitivity in MDPs \cite{chow2014algorithms,bisi2022risk,brown2020bayesian,bauerle_2011_markov}, mainly due to a delicate control it offers for the undesirable tail of return distribution. ERM, as another popular notion, has long been considered for risk-sensitive control in MDPs and RL \cite{howard1972risk,borkar2002risk,hau2023entropic,hu2023tighter,fei2020risk}. However, much of the existing literature focuses on undiscounted settings, despite the prevalence of discounted MDPs; see, e.g., \cite{bauerle_more_2014,hau2023entropic,mihatsch2002risk} for notable exceptions.

In this paper, we study risk-sensitive RL for discounted MDPs, assuming access to a generative model of the MDP, which is a simulator that generates samples from the true MDP for arbitrary state-action pairs. We consider a broad and important class of risk measure that includes risk measures that can be expressed as an Optimized Certainty Equivalent (OCE) \cite{ben2007old}. We refer to Section \ref{Section:Background} for definitions, and to Appendix \ref{Section:Appendix:RiskMeasures} for a more detailed primer on risk measures. 
The class of OCE measures captures many important risks such as CVaR, ERM, and mean-variance, as special cases (that are derived under suitable choices of utility functions).


In risk-sensitive RL, objectives can be formulated in two ways: In the \emph{non-recursive} (also called non-iterated or static) formulation, the risk measure is directly applied to the total return 
\cite{borkar2002q,borkar2002risk,hau2023dynamic}, 
while in the \emph{recursive} formulation (also called iterated, nested, or dynamic), the risk measure is applied at every step $t$ to the reward-to-go \cite{asienkiewicz2017note,bauerle_markov_2022,deng2025near}. 
The {\color{black}non-}recursive approach may allow the agent to visit high-risk states, even though the risk of the entire trajectory is still controlled, which might be unacceptable in many safety-critical applications. 
In contrast, the recursive approach may lead to a more cautious behavior by controlling risk at every step, which can be either desirable or overly conservative depending on the application \cite{deng2025near,xu2023regret,wang2025reductions}. Due to these qualitative differences, the two are considered as orthogonal modeling choices. 
From a technical standpoint, a key distinction is that non-recursive formulations do not generally admit Bellman-type optimality equations and may result in time-inconsistent optimal policies (see \cite{jaquette1976utility}), whereas recursive formulations preserve Bellman-type optimality structures. Motivated by these considerations, we study risk-sensitive discounted RL with objectives defined via the recursive OCE.

\subsection{Main Contributions}

We consider risk-sensitive RL in tabular discounted MDPs under 
recursive OCE in the generative setting. Learning performance is assessed in terms of sample complexity, defined as the total number $T$ of samples required, for given $(\varepsilon,\delta)$, to obtain either an $\varepsilon$-optimal policy (the \emph{policy learning} problem) or an $\varepsilon$-close approximation of the optimal Q-value in the max-norm (the \emph{value learning} problem), with probability exceeding $1-\delta$.

We make the following contributions. We propose a model-based algorithm, called Model-Based OCE Value Iteration (MB-\OCEVI), and establish PAC-type bounds on its sample complexity for both value learning and policy learning (Theorem \ref{Theorem:UpperBounds}), under the assumption that the utility function $u$ defining the OCE belongs to $\U_1$, which is essentially the set of utilities with a full domain ---for a more precise definition and context, see Subsection \ref{sec:OCE}. 
These bounds have an optimal dependence on (up to log-factors) on the size of state-action space, $SA$, and hold simultaneously for all OCE objectives defined using utilities in $\U_1$.  
%

We further show that the restriction to OCEs associated to utilities $u\in \U_1$ is necessary. We prove this claim by establishing impossibility results on the PAC sample complexity bound for when OCE is defined by a utility function $u\notin \U_1$. We thus provide an exact characterization of OCE measures, which are PAC learnable under value and policy learning. 

Finally, we establish worst-case lower bounds on the sample complexity under OCEs. For each problem, we present two lower bounds. The first one (Theorems \ref{theorem:genericValue} and \ref{theorem:genericPolicy}) is a general lower bound that holds for any OCE defined by a utility $u\in \U_1$. These lower bounds has an optimal dependence  (up to log-factors) on $S,A,\frac{1}{\delta},\varepsilon$, but a complicated dependence on $1/(1-\gamma)$. 
The second lower bound (Theorems \ref{theorem:specialValue} and \ref{theorem:specialPolicy}) has an explicit dependence on $1/(1-\gamma)$, but holds for a sub-class of utilities in $\U_1$. 
Nevertheless, we show that this sub-class includes all strongly risk-averse coherent OCE risk-measures that notably includes CVaR. We show that the latter lower bound can outperform the best existing lower bound for $\text{CVaR}_\tau$ in the regime where $\tau$ is very small. To the best of our knowledge, these results constitute the first upper and lower bounds on the sample complexity of recursive OCE in discounted MDPs, and are first impossibility results for RL under OCEs.

\subsection{Related Work}

\paragraph{Risk-neutral discounted RL.} There is a rich literature on provably-sample efficient learning algorithms in tabular discounted MDPs, encompassing a variety of settings such as  the generative setting \cite{kakade2003sample,gheshlaghi2013minimax,agarwal_model-based_2020,sidford2018variance,li2020breaking,jin2024truncated}, the offline (or batch) setting \cite{rashidinejad2021bridging,li2024settling}, and the online setting \cite{strehl_analysis_2008,lattimore_near-optimal_2014}. 
In the case of generative setting, which we also consider, early work includes \cite{kakade2003sample,kearns1998phased_Q_learning}, which was further improved and followed up by a ling of work, notably \cite{gheshlaghi2013minimax,sidford2018variance,wang2020randomized,li2024breaking,jin2024truncated}. 
Azar et al.~\cite{gheshlaghi2013minimax} provide the first minimax-optimal sample complexity bounds of $\widetilde{\cal O}\big(\frac{SA}{\varepsilon^2(1-\gamma)^3}\big)$ for both value learning and policy learning, albeit for substantially limited $\varepsilon$-ranges, which is attained by simple model-based methods. Further, they establish a lower bound of $\widetilde{\Omega}\big(\frac{SA}{\varepsilon^2(1-\gamma)^3}\big)$ for value learning. Model-free methods are presented in more recent subsequent work such as  \cite{sidford2018variance,wang2020randomized,jin2024truncated}. 
Notably, \cite{li2024breaking} has recently established an optimal bound valid for the entire  $\varepsilon$-range, using model-based methods built via the empirical MDP but with reward perturbations or conservative planning. We note that existing optimal sample complexities rely on techniques that crucially exploit the additivity of the return in terms of rewards; this structural property generally fails for risk-sensitive objectives, and the corresponding techniques do not carry over.

\paragraph{Risk-sensitive RL.} 
There exists a substantial literature on decision making under a risk measure in bandit and RL settings. In bandits, risk-sensitive objectives are typically studied through regret minimization; see, e.g., \cite{sani2012risk,maillard2013robust,khajonchotpanya2021revised}. Extensions to MDPs  introduce substantially richer structural and algorithmic challenges. 
The literature on RL under risk measures may be broadly categorized by the way the risk measure is applied (recursive vs.~non-recursive) as well as the type of risk measure studied. Representative examples include CVaR \cite{deng2025near,du2023provably,chen2024provably,lam2022risk}, ERM \cite{borkar2002risk,moharrami2025policy,marthe2023beyond,marthe2025efficient,hau2023entropic}, mean-variance risk \cite{sood2023deep,huang2022achieving,la2013actor}, and EVaR \cite{ni2022risk-evar,gangulyTMLRrisk-seeking}. Among these, CVaR has been the most extensively studied. 
Under recursive CVaR, \cite{deng2025near} analyzes sample complexity in the generative setting and provides a lower bound. 
%
Under recursive ERM, recent works such as \cite{fei2020risk,fei2021exponential,fei2021risk,hu2023tighter,liang2024regret} study online episodic RL in the regret setting. 
To the best of our knowledge, existing work on discounted MDPs under recursive ERM is limited to planning; a notable example is \cite{bauerle_markov_2022}, which provides a thorough theoretical treatment but does not propose learning algorithms.

There exists a line work in risk-sensitive RL and control that develop algorithms for an entire family of risk measures. Two notable families studied in this context are coherent risk measures and OCEs. While entropic risk  is not coherent, it belongs to the OCE; a brief overview of risk measures is provided in Appendix \ref{Section:Appendix:RiskMeasures}. 
Existing results for OCE risks \cite{wang2025reductions,xu2023regret,rigter2023one,lee2025risk} do not address provably sample-efficient learning under recursive entropic risk in discounted MDPs.  Furthermore, results for coherent risks \cite{petrik2012approximate,tamar2015policy,lam2022risk,zhao2024ra} do not apply to entropic risk.  In particular,   \cite{rigter2023one} considers offline RL in discounted MDPs under recursive OCE but does not provide sample-complexity guarantees. 
We also note that a connection between MDPs with recursive coherent risks and distributionally robust MDPs has been established in \cite{bauerle_markov_2022}. 


\section{Setting: RL under Recursive OCE}
\label{Section:Background}

\paragraph{Notations.} For $n \in \NN$, let $[n] := \{1, \ldots, n\}$. $\mathbbmss 1_A$ denotes the indicator function of an event $A$. Given a set $\cX$, $\Delta(\cX)$ denotes the probability simplex over $\cX$. We use the convention that $\|\cdot \| := \|\cdot \|_\infty$. 
We let $L^\infty(\Omega, \mathcal{F},\mathbb{P})$ denote the space of essentially bounded random variables on the probability space $(\Omega, \mathcal{F},\mathbb{P})$.

\subsection{Optimized Certainty Equivalents}
\label{sec:OCE}
For risk-averse agents it is natural to be able to rank different random variables based on risk-measures. In \cite{ben2007old} they propose the optimised certainty equivalent of a class of utility functions $\U_0$ to be defined shortly. This approach not only provides a direct link between microeconomic foundations and risk but also form a unifying framework as risk measures derived as OCEs are guaranteed to be convex and includes several well-known risk measures such as entropic risk, the mean-variance criterion, and conditional value-at-risk. 

Let $u:\R\rightarrow[-\infty,\infty)$ be a closed, proper concave, non-decreasing function satisfying that $u(0)=0$ and $1\in \partial u(0)$, where $\partial u$ denotes the superdifferential of $u$. Further, define $\text{dom}(u):=\{t\in\R|u(t)>-\infty\}$. The collection of all such functions are denoted by $\U_0$. A subset of great interest in this paper is the collection of finite utility functions, i.e. those where $\text{dom}(u)=\R$, which we denote by $\U_1$; that is, $\U_1=\{u\in\U_0|\text{dom}(u)=\R\}$. We further define the subclass of strongly risk-averse utility functions $\mathbb{U}_1^<:=\{u\in \U_1 |\forall t\neq 0:u(t)<t \}$.

The optimized certainty equivalent (OCE) of $u$ is the map $\OCE^u:L^\infty(\Omega,\mathcal{F},\P)\rightarrow \R$ defined as 
\begin{align*}
    \OCE^u(X) = \sup_{\eta\in \R}(\eta+\E[u(X-\eta)])\,.
\end{align*}
The notion of OCE was first introduced in \cite{ben2007old}, where they also show that the negative of the OCE is a convex risk-measure and that $\OCE^u(X)\leq \E[X]$. For brevity, we will often use the notation $\rho:=\OCE^u$ to refer to the OCE associated to the utility function $u$. 
The OCE $\rho$ admits the following properties: (i) $\rho(0)=0$ (normalization), (ii) $X\leq Y$ implies $\rho(X)\leq \rho(Y)$ (monotonicity), and (iii)  $\rho(c)=c$ for $c\in \R$ (consistency). 

{Let $\S$ be a finite set with size $S:=|\S|$, and $X$ be a random variable with support $\{v(s)\}_{s\in \S}$ with probabilities given by $\P(X=v(s))=p(s)$. We introduce the short-hand $\rho_p(v(s)):=\rho(X)$.}

\subsection{Discounted Markov Decision Processes and Recursive OCE Objectives}

A discounted Markov decision process (MDP) is a $5$-tuple $M=(\S,\A,P,R,\gamma)$, where $\S = \{1,2,\ldots,S\}$ is the finite state space of size $S:=|\S|$, $\A = \{1,2,\ldots,A\}$ is the finite action space of size $A:=|\A|$, $P:\S\times \A\rightarrow \Delta(\S)$ is the transition probability function, $R:\S\times \A\rightarrow [0,1]$ is the deterministic reward function, and $\gamma \in (0,1)$ is the discount factor. A stationary deterministic policy is a map $\pi:\S\rightarrow \A$. 
The agents interaction with the MDP $M$ is as follows. At initialization of the process, $M$ is in some initial state $s_0\in \S$. At each time $t\geq 0$, the agent is in state $s_t\in \S$ and decides on an action $a_t\in \A$. The MDP generates a reward $r_t:=R(s_t,a_t)$ and a next-state $s_{t+1}\sim P(\cdot|s_t,a_t)$. The MDP moves to $s_{t+1}$ when the next time slot begins, and this process continues ad infinitum. This process yields a growing sequence $(s_t,a_t,r_t)_{t\ge 0}$. 

The agent's goal is to maximize an objective function, as a function of the collected rewards $(r_t)_{t\ge 0}$, which depends on both $\gamma$ and $\rho$. In classical setting,  the value of a policy $\pi$ is defined as discounted sum of rewards collected under $\pi$. 
However, the classical objective fails to capture the inherent risk coming from the stochastic transitions and thus the uncertainty about the rewards collected during a trajectory. Under the recursive OCE criterion, the agent's objective is defined using the iteration of OCEs \cite{bauerle_markov_2024}. The value function $V^\pi$ and the state-action value function (or Q-value) $Q^\pi$ of a policy $\pi$ are defined informally as 
\begin{align*}
  V^\pi(s_0) & =r_0+\gamma\rho_{s_0,\pi(s_0)}\Big(r_1+\gamma\rho_{s_1,\pi(s_1)}\big(r_2 + \gamma \rho_{s_2,\pi(s_2)}(\ldots)\big)\Big)\,,
    \\
  Q^\pi(s_0,a) & =R(s_0,a)+\gamma\rho_{s_0,a}\Big(r_1+\gamma\rho_{s_1,\pi(s_1)}\Big(r_2 +\gamma \rho_{s_2,\pi(s_2)}(\ldots)\Big)\Big)\,
\end{align*}
where we used the convention that $\rho_{s,a}:=\rho_{P_{s,a}}$ and $r_t:=R(s_t,\pi(s_t))$. For all $(s,a)$, let $V^*(s):=\sup_{\pi}V^\pi(s)$ and $Q^*(s,a):=\sup_\pi Q^\pi(s,a)$ denote the optimal value (at state $s$) and Q-value (at state-action $(s,a)$), respectively, where $\sup$ is taken over set of all possible policies. Any policy $\pi^*$ satisfying $V^{\pi^*}=V^*$ is called optimal, and any policy $\pi$ obeying $V^\pi(s)\ge V^*(s) - \varepsilon$ for a given $\varepsilon>0$ is called $\varepsilon$-optimal. 
%
As established in \cite{bauerle_markov_2024}, there exists a stationary deterministic optimal policy $\pi^*:\cS\to\cA$ that achieves $V^*(s)$ for all states $s$ simultaneously, which is shown to satisfy the optimal Bellman equations:  
\begin{align*}
    V^*(s)& =\max_{a\in \A}\big(R(s,a)+\gamma \rho_{s,a}(V^*(s'))\big), \quad &&\forall s\in \S.
    \\
    Q^*(s,a)& =\!R(s,a)+\gamma\rho_{s,a}\big(\max_{a'\in \A}Q^*(s',a')\big), \quad &&\forall (s,a)\in \cS\times \cA.
\end{align*}
We introduce the optimal Bellman operator $\T:\R^{\S\times \A}\rightarrow \R^{\S\times \A}$ defined for $f:\S\times \A\rightarrow\R$ by 
\begin{align*}
    (\T f)(s,a) & :=R(s,a)+\gamma\rho_{s,a}\big(\max_{a'\in \A}f(s',a')\big)\,, \quad \forall (s,a)\in \cS\times \cA.
\end{align*}
It is evident that $Q^*$ is  the unique fixed-point of $\mathcal{T}$: $Q^*=\mathcal{T}Q^*$. 

Since $\mathcal{T}$ is a $\gamma$-contraction (Lemma \ref{lemma:OCE-contraction} in Appendix \ref{section:ConvergenceOfUVI}), it follows that $Q^*$ can be efficiently approximated arbitrarily well by a value-iteration-type algorithm (see Algorithm \ref{Alg:OCEVI}).

\subsection{Learning Performance}
Under a given OCE risk measure $\rho$ applied recursively, we consider RL algorithms that aim to find an $\varepsilon$-optimal policy or an $\varepsilon$-optimal value function for input $\varepsilon>0$. This is done by assuming access to a generative model (or simulator) of the MDP, which can produce a sample $s'\sim P_{s,a}$ for any queried state-action $(s,a)$. We consider two types of such algorithms, which we generically denote by $\mathcal U$: The first type outputs a $Q$-value $Q_T^{\mathcal{U}}:\S\times \A \rightarrow \R$, whereas the second outputs a policy $\pi_T^\mathcal{U}:\cS\rightarrow \A$ using $T$  samples.  

We evaluate the quality of an algorithm that outputs a $Q$-value by $\|Q^*-Q_T^\mathcal{U}\|$, and that outputs a policy by $\|V^*-V^{\pi_T^\mathcal{U}}\|$. Often, we will suppress $T$ from the notation. This leads to the notion of $(\varepsilon,\delta)$-correct value and policy for input parameters $(\varepsilon, \delta)$ as formalized below: 

\begin{definition}[$(\varepsilon,\delta)$-correct value and policy]
    An algorithm $\mathcal{U}$ that outputs a $Q$-value $Q^\mathcal{U}$ is called \emph{$(\varepsilon,\delta)$-value-correct} on a set of MDPs $\mathbb{M}$ if $\P(\|Q^*-Q^{\mathcal{U}}\|\leq \varepsilon)\geq 1-\delta$ for all $M\in \mathbb{M}$. Similarly, an algorithm $\mathcal{U}$ that outputs a policy $\pi^\mathcal{U}$ is called \emph{$(\varepsilon,\delta)$-policy-correct} on a set of MDPs $\mathbb{M}$ if $\P(\|V^*-V^{\mathcal{\pi^\mathcal{U}}}\|\leq \varepsilon)\geq 1-\delta$ for all $M\in \mathbb{M}$.
\end{definition}

The notion of $(\varepsilon,\delta)$-value-correctness yields a sample complexity notion in the case of \emph{value learning}, while $(\varepsilon,\delta)$-policy-correctness serves a similar role for \emph{policy learning}.

\section{Model-based Utility Value Iteration}
\label{Section:Model-basedVI}

We now present a simple model-based algorithm, called Model-Based OCE Value Iteration (MB-\OCEVI), for value and policy learning settings with the RL objective defined using recursive OCEs, assuming access to a generative model of the MDP. 

{
  


\begin{multicols}{2}
\footnotesize
\begin{algorithm}[H]

  \caption{Model estimation}
\label{alg:model-estimation}
\SetAlgoLined
\SetKwInput{KwInput}{Input}                
\SetKwInput{KwOutput}{Output}              
\DontPrintSemicolon

  {
  \KwInput{Generative model $P$}
  \KwOutput{Model estimate $\widehat{P}$}

  \SetKwFunction{FEstimateModel}{EstimateModel}
 
  \SetKwProg{Fn}{Function}{:}{}

\Fn{\FEstimateModel{$N$}}{
  $\forall$ $(s,z)\in \cS\times Z:$ $m(s,z) = 0$
  \;
  \For{each $z\in Z$}{
  \For{$i=1,2,\ldots,N$}{
  $s\sim P(\cdot|z)$
  \;
  $m(s,z) := m(s,z)+1$
  }
  $\forall s\in \cS: \widehat{P}(s,z) = \frac{m(s,z)}{N}$ 
  }
    \KwRet $\widehat{P}$
    }}
\end{algorithm}

\hfill

\footnotesize
\begin{algorithm}[H]
\label{Alg:OCEVI}
\SetAlgoLined
\SetKwInput{KwInput}{Input}                
\SetKwInput{KwOutput}{Output}              
\DontPrintSemicolon
  
  \KwInput{Empirical MDP $\widehat{M} = (\cS,\cA,\widehat{P},R,\gamma)$, OCE $\rho$, number of iterations $k$}
  \KwOutput{Estimate $Q_k$ of optimal $Q$-function $Q^*$}
  Initialization: $\forall (s,a)$ set $Q(s,a) =\frac{1}{2(1-\gamma)}$ \;
  \For{$j=0,1,\ldots,k-1$}{
  \For{all $(s,a)\in \cS\times \cA$}{
$Q_{j+1}(s,a) = R(s,a)+\gamma\rho_{s,a}(\max_{a'}Q(s',a'))$ \;
  }
  }
  $\forall s\in \cS:$ $\pi_k(s) = \argmax_{a\in\cA}Q_k(s,a)$ \;
  \KwRet $Q_k$ and $\pi_k$

  \caption{MB-OCE-VI}
\end{algorithm}
\end{multicols}

We introduce some necessary notations. Let $\widehat{P}$ denote the plug-in estimator of the transition function $P$, built using $N$ independent samples  from each state-action pairs of the MDP. More precisely, for $(s,a,s')\in Z\times \S$, $\widehat{P}(s'|s,a) = \frac{n(s,a,s')}{N}$, where $n(s,a,s')$ denotes the number of times $s'$ was observed under the pair $(s,a)\in Z$.  Further, let $\widehat M=(\S,\A,R,\widehat{P},\gamma)$ be the corresponding empirical MDP built using $\widehat{P}$ as described in algorithm \ref{alg:model-estimation}. 

We are now ready to introduce MB-\OCEVI whose pseduocode is provided as Algorithm \ref{Alg:OCEVI}. It is a value-iteration type algorithm that extends the classical value iteration for risk-neutral objectives to those defined recursively using the OCE $\rho$ applied to the empirical MDP $\widehat{M}$. The MB-\OCEVI\ algorithm works as follows. For any input $\varepsilon>0$, it  first collects $T=NSA$ samples, which is done by making $N$ calls to  the generative model, and then computes $\widehat{P}$. 
Then, it runs value-iteration updates which returns a policy $\pi_k$ and a Q-value estimate $Q_k$.  Finally, one can set $k=\log\big(\frac{1}{2\varepsilon(1-\gamma)}\big)/\log(1/\gamma)$ (see Lemma \ref{lemma:ConvergenceRateVI} in Appendix \ref{section:ConvergenceOfUVI}).


\section{Sample Complexity Analysis of MB-\OCEVI}
\label{sec:sample_complexity_analysis}

In this section, we present PAC bounds on the sample complexity of MB-\OCEVI\ for both value and policy learning. The bounds hold under the assumption that the OCE is defined by a utility $u\in \U_1$; we refer to Section \ref{sec:OCE} for the definition of $\U_1$. As it turns out, this restriction is necessary since the set $\U_1$ will be shown to be precisely the utilities that guarantee continuity of value-functions between similar MDPs on the same state-action space. For  $\gamma\in (0,1)$, we introduce $\effh:=\frac{1}{1-\gamma}$. 

\begin{theorem}
\label{Theorem:UpperBounds}
    Assume $u\in \U_1$. If the total number $T$ of calls to the generative model satisfies 
    \begin{align*}
        T\geq 32\frac{\gamma^2SA [u(-\effh)]^2}{\varepsilon^2(1-\gamma)^2}\log\bigg(\frac{8\gamma SAu'_+(-\effh)}{\varepsilon(1-\gamma)^2}\bigg)\,,
    \end{align*}
    then it holds that $\P(\|Q^*-Q_k\|\geq \varepsilon)\leq \delta$. Furthermore, if 
    \begin{align*}
        T\geq 128\frac{\gamma^4SA[u(-\effh)]^2}{\varepsilon^2(1-\gamma)^4}\log\bigg(\frac{16\gamma^2 SAu'_+(-\effh)}{\varepsilon(1-\gamma)^3}\bigg)\,,
    \end{align*}
    then it holds that $\P(\|V^*-V^{\pi_k}\|>\varepsilon)\leq \delta$. Here, $u'_+$ denotes the right derivative of $u$. 
\end{theorem}


Before we give the proof, we present in Table \ref{table:ComparisonUtilities} a comparison of the policy learning sample complexity upper bounds for a set of specific risk-measures (see Appendix \ref{Section:Appendix:RiskMeasures}).

\subsection{Comparison with Existing Sample Complexity Bounds} 

Here we provide a comparison between the sample complexity of policy learning in Theorem \ref{Theorem:UpperBounds} to best existing bounds for some concrete cases of OCE in the recursive setting. 

\paragraph{CVaR.} For CVaR with parameter $\tau\in (0,1)$, Theorem \ref{Theorem:UpperBounds} gives a policy learning sample complexity of $\widetilde{\mathcal{O}}\Big(\frac{SA}{\varepsilon^2(1-\gamma)^6\tau^2} \Big)$. For recursive CVaR, the best existing bound is due to \cite{deng2025near}, and scales as $\widetilde{\mathcal{O}}\Big(\frac{SA}{\varepsilon^2(1-\gamma)^4\tau^2} \Big)$. This shows that our general analysis leaves a gap of $\frac{1}{(1-\gamma)^2}$. The analysis in \cite{deng2025near} exploits the specific properties of CVaR that cannot be generalized to a generic OCE. We also mention that there is no specialized result for value learning for CVaR, to the best of our knowledge.

\paragraph{Entropic.} For the entropic risk with parameter $\beta>0$, the best available sample complexities for value learning and policy learning are reported in \cite{mortensen2025entropic}. For value learning, Theorem \ref{Theorem:UpperBounds} gives a bound of $\widetilde{\mathcal{O}}\Big(\frac{SA}{\varepsilon^2(1-\gamma)^3\beta} \big(e^{\frac{\beta}{1-\gamma}}-1\big)^2\Big)$, which is worse by a factor of $1/(1-\gamma)$ compared to the bound in \cite{mortensen2025entropic}, which scales as $\widetilde{\mathcal{O}}\Big(\frac{SA}{\varepsilon^2(1-\gamma)^2\beta^2} \big(e^{\frac{\beta}{1-\gamma}}-1\big)^2\Big)$. For policy learning, the resulting bound from Theorem \ref{Theorem:UpperBounds} scales as 
$\widetilde{\mathcal{O}}\Big(\frac{SA}{\varepsilon^2(1-\gamma)^5\beta} \big(e^{\frac{\beta}{1-\gamma}}-1\big)^2\Big)$, which is again off by a factor of $1/(1-\gamma)$ compared to the corresponding bound in \cite{mortensen2025entropic}. We note that the bounds in \cite{mortensen2025entropic} uses some proof elements that are specifically tailored to the entropic measure, and cannot be applied for generic OCEs.

\begin{remark}
    We show in Proposition \ref{proposition:DerivativeIsDominatedByU} that $u'_+(-\frac{\gamma}{1-\gamma})\leq |u(-\effh)|$ for piecewise differentiable $u\in\U_1$ and since $|u(-\effh)|\geq \effh$, the dominating term for the effective horizon $\effh$ in the bound is $|u(-\effh)|$. For the entropic risk, this is exponentially larger than $\effh$, while for CVaR and the mean-variance criterion, it contributes with polynomial factors of $\effh$. 
\end{remark}

\begin{table*}[!t]
    \centering
    {\footnotesize
    \begin{tabular}{c|c|c|c}
    \toprule
        {Name} & {$\OCE(X)$} & {Utility $u(t)$} & {Sample Complexity (Policy Learning)}\\ \midrule

       Entropic, $\beta>0$ & $-\frac{1}{\beta}\log(\E[e^{-\beta X}])$ & $\frac{1}{\beta}(1-e^{-\beta t})$ & $\widetilde{\mathcal{O}}\Big(\frac{SA}{\varepsilon^2(1-\gamma)^5\beta}\big(e^{\frac{\beta}{1-\gamma}}-1\big)^2\Big)$
         \\ \hline
         CVaR, $\tau\in (0,1)$ & $\text{CVaR}_\tau(X)$ & $\big[\frac{t}{\tau}\big]^-$ & $\widetilde{\mathcal{O}}\Big(\frac{SA}{\varepsilon^2(1-\gamma)^6\tau^2} \Big)$
         \\ \hline
        Mean-variance &  $\E[X]-\frac{1}{2}\text{Var}(X)$ & $\begin{cases}
             t-\frac{1}{2}t^2 \,\, &t\leq 1 
             \\
             \frac{1}{2} \,\, &t>1
         \end{cases}$ & $\widetilde{\mathcal{O}}\Big(\frac{SA}{\varepsilon^2(1-\gamma)^8} \Big)$
         \\ \hline
       Essential Infimum &   $\text{Essinf}(X)$ & $\begin{cases}
             0 \quad &t\geq 0
             \\
             -\infty \quad &t<0
         \end{cases}$ & $\infty$
         \\
         \bottomrule
    \end{tabular}
    }
    \caption{Implication of Theorem\ref{Theorem:UpperBounds} for different risk measures}
    \label{table:ComparisonUtilities}
\end{table*}

\subsection{Proof of Theorem \ref{Theorem:UpperBounds}}

In this section, we prove Theorem \ref{Theorem:UpperBounds}. First, we present a lemma, proven in Appendix \ref{app:missing_proofs}, that characterizes the smoothness of Q-values under the OCE defined by a utility $u\in \mathbb U_1$, when the transition function is perturbed. This result could of interest, beyond the considered RL setting. 

\begin{restatable}{lemma}{OCESimLemma}
\label{lemma:SimulationLemma}
    Let $M=(\S,\A,P,R,\gamma)$ and $\widetilde{M}=(\S,\A,\widetilde{P},R,\gamma)$ be two MDPs that only differ in their transition function and $\pi $ a fixed stationary policy and $u\in \U_1$ be a utility function. Then 
    \begin{align*}
\|Q^\pi-\widetilde{Q}^\pi\|\leq \frac{\gamma}{1-\gamma}\max_{s,a}\sup_{\eta \in [0,\effh]}\bigg| \sum_{s'\in \S}[P_{s,a}(s')-\widetilde{P}_{s,a}(s')]u(V^\pi(s')-\eta)\bigg|\,.
\end{align*}
\end{restatable}

\begin{proof}[Proof of Theorem \ref{Theorem:UpperBounds}]

Let $\varepsilon>0$. Let $Q_k$ be the $Q$-function output of \OCEVI\ after $k$ iterations, and let 
$\widehat{Q}^*$ denote the optimal Q-value in $\widehat{M}$. Note that $\widehat{Q}^{\pi^*}$ is the Q-value in the empirical MDP of the optimal policy of the true MDP. Using a standard decomposition which uses $\widehat{Q}^*\geq \widehat{Q}^{\pi^*}$, we have for any $(s,a)\in \cS\times \cA$,
    \begin{align}
        Q_k(s,a)&= Q^*(s,a) + Q_k(s,a) - \widehat{Q}^*(s,a)+\widehat{Q}^*(s,a)-Q^*(s,a)
        \nonumber\\
        & \geq Q^*(s,a) + Q_k(s,a) - \widehat{Q}^*(s,a) + \widehat{Q}^{\pi^*}(s,a)-Q^*(s,a)
        \nonumber\\
        & \geq Q^*(s,a) -\|Q_k-\widehat{Q}^*\| - \|\widehat{Q}^{\pi^*}-Q^*\|\, . \label{equation:decomposition_Qvalue}
    \end{align}

Therefore, to establish $\varepsilon$-value-correctness it suffices to ensure $\| \widehat{Q}^{\pi^*}-Q^* \|\le \varepsilon/2$ and $\|Q_k-\widehat{Q}^*\|\le \varepsilon/2$. 
By Lemma \ref{lemma:ConvergenceRateVI}, we can have $\|Q_k - \widehat{Q}^* \|<\varepsilon/2$ by picking $k\geq \log(\frac{1}{(1-\gamma)\varepsilon})/\log(1/\gamma)$. 

To control $\|\widehat{Q}^{\pi^*}-Q^*\|$, we apply Lemma \ref{lemma:SimulationLemma} with $\widetilde{M} = \widehat{M}$ and $\pi = \pi^*$, which yields 
\begin{align*}
    \|Q^* - \widehat{Q}^{\pi^*}\|\leq \frac{\gamma}{1-\gamma}\max_{s,a}\sup_{\eta\in[0,\effh]}\bigg|\sum_{s'}[P_{s,a}(s')-\widehat{P}_{s,a}(s')]u(V^*(s')-\eta) \bigg|\,.
\end{align*}
For a fixed $(s,a)$, let $\eta^*$ be an arbitrary but fixed optimizer of $$
    \sup_{\eta \in [0,\effh]}\bigg|\sum_{s'\in \S}(P_{s,a}(s')-\widehat{P}_{s,a}(s'))u(V^\pi(s')-\eta)\bigg|.$$ 
Since $\eta^\star$ 
is data-dependent, a direct application of Hoeffding's inequality is not allowed. To handle this, we discretize the interval $[0,\effh]$. Let $D$ denote the corresponding discretized set built using uniform discretization. The following lemma controls the introduced error, which establishes that $u$ is sufficiently regular to get a handle on the number of discretization points needed:

\begin{restatable}{lemma}{BoundingSimulationTermCombined}
\label{lemma:BoundingSimulationTerm-combined}
    Let $u\in \U_1$, and define $\bar\eta = \min_{\eta\in D}u'_+(-\effh)|\eta^*-\eta|$.  
    If the set $D$ of discretization points satisfies $|D|\geq \frac{u'_+(-\effh)}{\varepsilon(1-\gamma)}$, then
    \begin{align*}
    \max_{s,a}\bigg|\sum_{s'}(P_{s,a}(s')-&\widehat{P}_{s,a}(s'))u(V^*(s')\!-\!\eta^*)\bigg|\leq  \max_{s,a}\bigg|\sum_{s'}(P_{s,a}(s')-\widehat{P}_{s,a}(s'))u(V^*(s')\!-\!\bar{\eta})\bigg|+ \frac{\varepsilon}{2}\,.
    \end{align*}
\end{restatable}

Lemma \ref{lemma:BoundingSimulationTerm-combined} implies that 
    \begin{align}
           \P\bigg(\max_{s,a}\sup_{\eta \in [0,\effh]}\sum_{s'\in \S}[P_{s,a}(s')-\widehat{P}_{s,a}(s')]u(V^*(s')-\eta)>\varepsilon\bigg) \leq 
           \nonumber\\
           \P\bigg(\exists\bar{\eta}\in D:\max_{s,a}\sum_{s'\in \S}[P_{s,a}(s')-\widehat{P}_{s,a}(s')]u(V^*(s')-\eta)>\frac{\varepsilon}{2}\bigg)\,.\label{eq:prob_eta}
    \end{align}
An application of Hoeffding's inequality (Lemma \ref{lemma:Hoeffding} in the appendix) and taking a union bound over $D$ and all state-action pairs, it follows that the right-hand side of (\ref{eq:prob_eta}) will be smaller than $\delta$ if any state-action pair is sampled 
 $N=\frac{8[u(-\effh)]^2}{\varepsilon^2}\log\big(\frac{4SAu'_+(-\effh)}{\varepsilon(1-\gamma)}\big)$
times. After adjusting $\varepsilon$ appropriately, it follows that if
\begin{align*}
            T\geq 32\frac{\gamma^2SA[u(-\effh)]^2}{\varepsilon^2(1-\gamma)^2}\log\bigg(\frac{8\gamma SAu'_+(-\effh)}{\varepsilon(1-\gamma)^2}\bigg),
\end{align*}
then $\P(\|\widehat{Q}^{\pi^*}-Q^*\|\geq \frac{\varepsilon}{2})<\delta$, showing the first part of the theorem.  

To prove the second result, we use the followig lemma, which is proven in the appendix:
\begin{restatable}{lemma}{greedyPolicyBound}
    \label{lemma:DeteriorationOfGreedyPolicy}
        Let $\varepsilon>0$. Let $\overline{V}\in \R^S$ be a value function obeying $\|V^*-\overline{V}\|<\varepsilon$, and  $\pi^G:=\argmax_a[R(s,a)+\gamma \rho_{s,a}(\overline{V}(s'))]$ be a greedy policy with respect to $\overline{V}$. Then, 
    $\|V^*-V^{\pi_G}\|\leq \frac{2\gamma}{1-\gamma}\varepsilon$\,.
\end{restatable}
Applying Lemma \ref{lemma:DeteriorationOfGreedyPolicy} with $\overline V = \widehat{V}^{\pi_k}$, we have shown that $\|\widehat{V}^{\pi_k} - V^*\|\le \varepsilon$ with probability $1-\delta$. Further, note that $\pi_k$ by construction is the greedy policy with respect to $\widehat{V}^{\pi_k}$. Therefore, the true value of $\pi_k$ satisfies, with probability at least $1-\delta$,
 $$
 \|V^*-V^{\pi_k}\|\leq \frac{2\gamma}{1-\gamma}\|Q^*-Q_k\|,
 $$
 which after properly adjusting $\varepsilon$ yields the announced sample complexity for policy learning. 
\end{proof}

\section{Impossibility Results}

In this section, we present some results for PAC-learnability under OCEs. Specifically, we establish  that for essentially all utility functions $u\!\in\!\U_0$ for which $\text{dom}(u)\!\neq\!\R$, it is impossible to obtain PAC-bounds for the corresponding learning problems with $\text{OCE}^u$. This is done by constructing a parametric family of  simple MDPs for which the value functions are not continuous in the parameter. By making this parameter sufficiently small, we can thus have two MDPs with a large gap in value functions, while the number of samples it takes to distinguish them can be made arbitrarily large. 

We will have to require that $u$ is not the identity for $u\geq 0$. A formalized in Proposition \ref{proposition:expectationutilities}, the only utilities this assumption rules out all lead to the same risk measure, namely the expectation.

\begin{theorem}[{Impossibility, value learning}]
\label{theorem:ImpossibilityValue}
    Let $u\in\U_0$ be a utility function for which (i) $\text{dom}(u)\neq \R$ and (ii) there is some $t_0>0$ such that $u(t_0)<t_0$. Then, there exists a class of MDPs $\M$ with $S$ states, $A$ actions, and discount factors $\gamma$ such that no value-learning algorithm $\mathcal{U}$ can be $(\varepsilon,\delta)$-correct on $\M$. 
\end{theorem}

\begin{proof}
We consider a class $\M$ of MDPs with a single action $a$ and three states $s_0,s_G$ and $s_B$, where $s_G$ and $s_B$ are absorbing and $R(s_G,a)=1$ and $R(s_B,a)=R(s_0,a)=0$. Finally, from $s_0$ it is possible to transition to $s_G$ with probability $p$, and to $s_B$ with probability $1-p$. The MDPs in $\M$ differ only in their value of $p$. Hence, we can parametrize them by $p$ and write $M_p$ to represent the MDP in which  transition probability to $s_G$ is $p$. Let $Q_1$ and $Q_p$ denote the Q-values of $M_1$ and $M_p$, respectively. We have 
\begin{align*}
    Q_1(s_0,a) = \frac{\gamma}{1-\gamma}, \qquad 
    Q_p(s_0,a)  = \gamma \sup_{\eta \in [0,\effh]}\Big\{\eta + pu(\effh-\eta)+(1-p)u(-\eta)\Big\}\,.
\end{align*}
By picking $\gamma$ large enough, we get that $\frac{1}{1-\gamma}-\xi>t_0$, where $\xi=:-\inf \text{dom}(u)$. Since for any $\eta >\xi$, we have that $\eta+u(\effh-\eta)+(1-p)u(-\eta)=-\infty$, it holds that 
\begin{align*}
    \sup_{\eta \in [0,\effh]}\Big\{\eta + pu(\effh-\eta)+(1-p)u(-\eta)\Big\} = \sup_{\eta \in [0,\xi]}\Big\{\eta + pu(\effh-\eta)+(1-p)u(-\eta)\Big\}.
\end{align*}
Moreover, since 
\begin{align*}
    \sup_{\eta \in [0,\effh]}\Big\{\eta + pu(\effh-\eta)+(1-p)u(-\eta)\Big\} & \leq \sup_{\eta \in [0,\effh]}\Big\{\eta + pu(\effh-\eta)\Big\}
    \\
    & = \xi+pu(\effh-\xi)  <\effh,
\end{align*}
it follows that $Q_1(s_0,a)-Q_p(s_0,a)\geq \gamma(\effh-\xi)=:\Delta>0$ for all $p<1$. Introduce
\begin{align*}
    \mathcal{E}_p := \{|Q_p^*(s_0,a)-Q^\mathcal{U}(s_0,a)|\leq \varepsilon \}, \qquad \quad 
    \mathcal{E}_1 & := \{|Q_1^*(s_0,a)-Q^\mathcal{U}(s_0,a)|\leq \varepsilon \}.
\end{align*}
By picking $\varepsilon<\Delta/2$, the output $Q^\mathcal{U}$ of any $(\varepsilon,\delta)$-correct algorithm $\mathcal{U}$ satisfies that $\mathcal{E}_1\cap \mathcal{E}_p = \emptyset$, which implies that $Q^\mathcal{U}$ can never be $\varepsilon$-correct on both MDPs simultaneously. 

The next argument relies on a change-of-measure between the one induced by the two MDPs. Since under $\mathbb{P}_1$ the only possible sequence of transitions from trying action $a$ in $s_0$ is to observe a transition to $s_G$ every time, the probability of its occurrence under $\mathbb{P}_p$ is $p^N$, where $N$ is the number of samples. Assuming that $\delta <\frac{1}{2}$, we then have for any $(\varepsilon,\delta)$-correct algorithm $\mathcal{U}$ that
\begin{align*}
   \mathbb{P}_p(\mathcal{E}^\complement_p) \geq   \mathbb{P}_p(\mathcal{E}_1) = \E_p[\mathbbmss 1_{\mathcal{E}_1}] = p^N\E_1[\mathbbmss 1_{\mathcal{E}_1}] = p^N \P_1(\mathcal{E}_1)\geq p^N(1-\delta)\geq \frac{1}{2}p^N\,.
\end{align*}
Solving the equation $\frac{1}{2}p^N\geq \delta$ for $N$, we find that if 
    $N\leq \frac{\log(1/2\delta)}{\log(1/p)}$, 
then any algorithm that is $(\varepsilon,\delta)$-correct on $M_1$ cannot also be $(\varepsilon,\delta)$-correct on $M_p$. Since this expression tends to infinity as $p\rightarrow 1$, the result follows. 
\end{proof}


\begin{restatable}[{Impossibility, policy  learning}]{theorem}{ImpossibilityPolicy}
\label{theorem:ImpossibilityPolicy}
    Let $u\in\U_0$ be a utility function for which (i) $\text{dom}(u)\neq \R$ and (ii) there is some $t_0>0$ such that $u(t_0)<t_0$. Then, there exists a class of MDPs $\M$ with $S$ states, $A$ actions, and discount factors $\gamma$ such that no policy-learning algorithm $\mathcal{U}$ can be $(\varepsilon,\delta)$-correct on $\M$. 
\end{restatable}

The proof of Theorem \ref{theorem:ImpossibilityPolicy} is quite similar to that of Theorem \ref{theorem:ImpossibilityValue} and is postponed to the appendix. The following result, proven in the appendix, shows that the assumption in Theorems \ref{theorem:ImpossibilityValue}-\ref{theorem:ImpossibilityPolicy} only excludes the utilities that correspond to $\OCE_u\equiv \E$. 

\begin{proposition}
\label{proposition:expectationutilities}
    Let $u\in\U_0$ be any utility function for which $u(t)=t$ for all $t\geq 0$. Then $\OCE_u(X)=\E[X]$.
\end{proposition}

In summary, the upper bounds and the above impossibility results now classify exactly the class of utility functions for which the worst-case sample complexities are finite. These are precisely the utility functions in $\U_1$, i.e., the ones with full domain $\text{dom}(u)=\R$, apart from the special utilities $u$ for which $\text{OCE}^u=\E$ (hence, risk-neutral problem).

\begin{remark}
   Inspecting the class of MDPs used in our construction, we note that it is straightforward to extend the same impossibility results to the case of RL problems with objectives defined using non-recursive OCEs (see Chapter 5 in \cite{bauerle_markov_2024}). This is due to the fact that the construction allows for direct computation of the value functions in this case. 
\end{remark}

\section{Lower Bounds}


In this section, we provide sample complexity lower bounds for both policy and value learning. For each problem, we present two lower bounds. The first one (Theorems \ref{theorem:genericValue} and \ref{theorem:genericPolicy}) is a general lower bound that holds for all $u\in \U_1$. This lower bounds has an optimal dependence  (up to log-factors) on $S,A,\frac{1}{\delta},\varepsilon$, but a complicated dependence on $u$ and $\effh$. 
The second lower bound (Theorems \ref{theorem:specialValue} and \ref{theorem:specialPolicy}) enjoys a more interpretable dependence on $\effh$, but holds for a sub-class of utilities in $\U_1$. 
The sub-class includes all strongly risk-averse coherent OCE risk-measures including $\text{CVaR}_\tau$.

The MDP constructions used in the proofs are similar to \cite{mortensen2025entropic}, where only the entropic risk is considered. The main challenge with general OCE is to get tight lower bounds on the difference in value functions for the MDP building blocks with similar transitions. These building blocks are shown in Figure \ref{fig:Lowerbound_maintext} (a) for value learning and (b) for policy learning. 

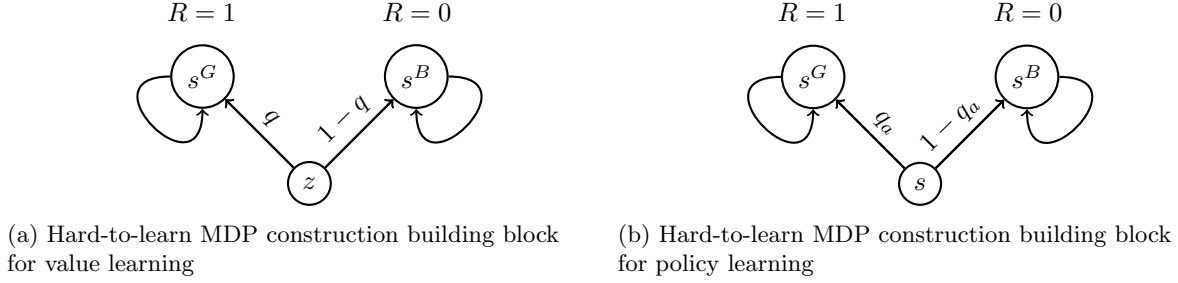
\begin{figure}[h]
\begin{subfigure}{.5\textwidth}
\centering
\begin{tikzpicture}[node distance={20mm}, thick, main/.style = {draw, circle}]
\node[main] (1) {$z$}; 
\node[main] (2) [above left of=1] {$s^G$};
\node[main] (3) [above right of=1] {$s^B$};
\draw[->] (2) to [out=180,in=270,looseness=5] (2);
\draw[->] (3) to [out=0,in=270,looseness=5] (3);
\draw[->] (1) -- node[midway, above right, sloped, pos=0.7] {$q$} (2);
\draw[->] (1) -- node[midway, above left, sloped, pos=0.9] {$1-q$} (3);
\draw[->] (1) -- node[midway, above right, pos=2] {$R=1$} (2);
\draw[->] (1) -- node[midway, above left,  pos=2] {$R=0$} (3);
\end{tikzpicture} 
\caption{Hard-to-learn MDP construction building block \\for value learning}
\end{subfigure}
\begin{subfigure}{.5\textwidth}
\centering
\begin{tikzpicture}[node distance={20mm}, thick, main/.style = {draw, circle}]
\node[main] (1) {$s$}; 
\node[main] (2) [above left of=1] {$s^G$};
\node[main] (3) [above right of=1] {$s^B$};
\draw[->] (2) to [out=180,in=270,looseness=5] (2);
\draw[->] (3) to [out=0,in=270,looseness=5] (3);
\draw[->] (1) -- node[midway, above right, sloped, pos=0.7] {$q_a$} (2);
\draw[->] (1) -- node[midway, above left, sloped, pos=0.9] {$1-q_a$} (3);
\draw[->] (1) -- node[midway, above right, pos=2] {$R=1$} (2);
\draw[->] (1) -- node[midway, above left,  pos=2] {$R=0$} (3);
\end{tikzpicture} 
\caption{Hard-to-learn MDP construction building block \\for policy learning}
\end{subfigure}
\caption{Lower bound building blocks for value learning (a) and policy learning (b)}
\label{fig:Lowerbound_maintext}
\end{figure}

The idea behind the proofs is to first obtain a tight lower bound on the difference in value functions for two instances with almost similar transition probabilities to ensure only the optimal policy is $\varepsilon$-good for policy learning and that an output $Q^\mathcal{U}$ cannot be $\varepsilon$-good on both instances for value learning. Then one has to lower bound the number of samples needed for the learner to correctly identify which instance the samples are from which is harder the more similar the instances are. Finally, one has to combine the building blocks to get the dependence on $S$ and $A$ and then perform some parameter tuning to make the learning as hard as possible.  

\subsection{Value Learning Lower Bounds}

\begin{restatable}{theorem}{LowerboundGenericValue}
\label{theorem:genericValue}
    Let $u\in \U_1$ be a utility function. There exist $\bar{\varepsilon}(u,\effh)>0$, constants $c_1,c_2>0$, and a strictly positive function $\Phi:\U_1\times[1,\infty)\rightarrow (0,\infty)$ such that for any RL algorithm $\mathcal{U}$ that outputs a $Q$-value $Q^\mathcal{U}$, any $\delta\in (0,\frac{1}{4})$, and $\varepsilon\in \big(0,\bar{\varepsilon})$, 
    it holds: if the total number $T$ of samples satisfies  
\begin{align*}
    T \leq \frac{SA\Phi(u,\effh)}{c_1 \varepsilon^2}\log\Big(\frac{SA}{c_2\delta}\Big),
\end{align*}
then there exists some MDP $M$ with $S$ states and $A$ actions for which $\P(\|Q^*_M-Q_T^{\mathcal{U}}\|>\varepsilon)\geq \delta$.
\end{restatable}

\begin{restatable}{theorem}{LowerboundSpecificValue}
\label{theorem:specialValue}
    Let $u\in \U_1$ be a utility function for which $u_+'(0)<1<u'_-(0)$. There exist $\bar{\varepsilon}(u,\effh)>0$ and constants $c_1,c_2>0$ such that for any RL algorithm $\mathcal{U}$ that outputs a $Q$-value $Q^\mathcal{U}$, any $\delta\in (0,\frac{1}{4})$, and $\varepsilon\in \big(0,\bar{\varepsilon})$, the following holds: if the total number $T$ of samples satisfies  
\begin{align*}
    T \leq \frac{SA }{c_1 \varepsilon^2}\log\Big(\frac{SA}{c_2\delta}\Big)|u(-\effh)|^2\bigg( 1-\bigg[\frac{1-u'_+(0)}{u'_+(-\effh)-u'_+(0)}\bigg]^2\bigg),
\end{align*}
then there exists some MDP $M$ with $S$ states and $A$ actions for which $\P(\|Q^*_M-Q_T^{\mathcal{U}}\|>\varepsilon)\geq \delta$.
\end{restatable}


\subsection{Policy Learning Lower Bounds}


\begin{restatable}{theorem}{LowerboundGenericPolicy}
\label{theorem:genericPolicy}
    Let $u\in \U_1$ be a utility function. There exist $\bar{\varepsilon}(u,\effh)>0$, a strictly positive function $\Phi:\U_1\times[1,\infty)\rightarrow (0,\infty)$ and constants $c_1,c_2>0$ such that for any RL algorithm $\mathcal{U}$ that outputs a policy $\pi_\mathcal{U}$, any $\delta\in (0,\frac{1}{4})$, and $\varepsilon\in (0,\bar{\varepsilon})$, 
    we have: if the total number $T$ of samples satisfies  
\begin{align*}
    T \leq \frac{SA\Phi(u,\effh)}{c_1 \varepsilon^2}\log\Big(\frac{S}{c_2\delta}\Big),
\end{align*}
then there exists some MDP $M$ with $S$ states and $A$ actions for which $\P(\|V^*_M-V^{\pi_\mathcal{U}}\|>\varepsilon)\geq \delta$.
\end{restatable}

\begin{restatable}{theorem}{LowerboundSpecificPolicy}
\label{theorem:specialPolicy}
    Let $u\in \U_1$ be a utility function for which $u_+'(0)<1<u'_-(0)$. There exist $\bar{\varepsilon}(u,\effh)>0$ and constants $c_1,c_2>0$ such that for any RL algorithm $\mathcal{U}$ that outputs a policy $\pi_\mathcal{U}$, any $\delta\in (0,\frac{1}{4})$, and $\varepsilon\in \big(0,\bar{\varepsilon})$, the following holds: if the total number $T$ of samples satisfies  
\begin{align*}
    T \leq \frac{SA }{c_1 \varepsilon^2}\log\Big(\frac{S}{c_2\delta}\Big)|u(-\effh)|^2\bigg( 1-\bigg[\frac{1-u'_+(0)}{u'_+(-\effh)-u'_+(0)}\bigg]^2\bigg),
\end{align*}
then there exists some MDP $M$ with $S$ states and $A$ actions for which $\P(\|V^*_M-V^{\pi_\mathcal{U}}\|>\varepsilon)\geq \delta$.
\end{restatable}

\subsection{Discussion of Bounds}
The presented lower bounds establish the first sample complexity lower bounds for the family of OCEs, to our best knowledge. For specific OCE measures, there are two relevant lower bounds in the literature: the one in \cite{deng2025near} derived for CVaR (policy learning), and the one in \cite{mortensen2025entropic} for the entropic risk (value and policy learning). The lower bound in \cite{mortensen2025entropic} has an exponential dependence on $\effh$. The dependence on $\effh$ in our lower bound is not straightforward, which makes the comparison difficult. However, it could be that our bound is not as sharp as those in \cite{mortensen2025entropic}, but they hold for a much larger set of problems. A comparison with the bound in \cite{deng2025near} is provided later. 

Theorems  \ref{theorem:genericValue} and \ref{theorem:genericPolicy} establish  that our upper bounds from Theorem \ref{Theorem:UpperBounds} are tight in $S,A,\frac{1}{\varepsilon}, \frac{1}{\delta}$ but with a dependence on $\effh$ that is not interpretable. 

While the assumption that $u'_+(0)<1<u'_-(0)$ excludes entropic risk and the mean-variance criterion, it is inclusive enough to include all finite strongly risk-averse coherent risk-measures, which by Theorem 3.1 of \cite{ben2007old} are exactly the ones for which $u=\begin{cases}
    \lambda_1t\quad t\geq0
    \\
    \lambda_2t \quad t<0
\end{cases}$
for some $0\leq \lambda_1 < 1 < \lambda _2$. Furthermore, for any fixed finite strongly risk-averse coherent risk-measure, we obtain a lower bound scaling as $\widetilde{\Omega}\Big(\frac{SA}{\varepsilon^2(1-\gamma)^2} \Big)$. 
The risk-measure $\text{CVaR}_\tau$ is obtained by taking $\lambda_1 = 0, \lambda_2 = \frac{1}{\tau}$, yielding the following lower bound:

\begin{corollary}
\label{corr:LB_CVaR}
    For $\mathrm{CVaR}_\tau$, there exists constants $c_1,c_2>0$ such that for any RL algorithm $\mathcal{U}$ that outputs a policy $\pi_\mathcal{U}$, any $\delta\in (0,\frac{1}{4})$, and $\varepsilon\in (0,{\color{black}\bar{\varepsilon}})$, 
    if the total number $T$ of samples satisfies  
\begin{align*}
    T \leq \frac{SA }{c_1 \varepsilon^2(1-\gamma)^2\tau^2}\log\Big(\frac{S}{c_2\delta}\Big),
\end{align*}
then there exists some MDP $M$ with $S$ states and $A$ actions for which $\P(\|V^*_M-V^{\pi_\mathcal{U}}\|>\varepsilon)\geq \delta$. 
\end{corollary}

Corollary \ref{corr:LB_CVaR} implies a sample complexity lower bound of    $ \widetilde{\Omega}\Big(\frac{SA}{\varepsilon^2(1-\gamma)^2\tau^2}\Big)$ for policy learning under CVaR$_\tau$. This bound can be compared with the lower bound in \cite{deng2025near},  scaling as 
    $\widetilde{\Omega}\Big(\frac{(1-\gamma\tau)SA}{\varepsilon^2(1-\gamma)^4\tau} \Big)$, which to our knowledge constitutes the best existing lower bound for policy learning under CVaR. 
While the two bounds have the same dependence on $SA$ and $\log(1/\delta)$, they differ in terms of dependence on $\tau$ and the effective horizon $H_\gamma$. 
As $\tau\le 1$, our lower bound has a stronger dependence on $\tau$, while that of \cite{deng2025near} has a better dependence on the effective horizon. So neither bounds dominates the other uniformly. The comparison depends on the relative scaling of $\tau$ and $H_\gamma$. If $\tau$ is sufficiently small (e.g., relative to $H_\gamma^2$), the $\tau$-dependence dominates and our lower bound has better overall scaling. But for large horizon relative to $\tau$, the lower bound of \cite{deng2025near} becomes larger. 


\section{Concluding Remarks}
\label{Section:Conclusion}

We studied the PAC-learnability and sample complexities of value learning and policy learning in finite discounted MDPs with a generative model, where the objective is defined recursively using a risk measure from the family of OCE measures. We introduced a simple model-based algorithm MB-\OCEVI\ and derived PAC bounds on the sample complexity when the utility function defining the OCE has full domain. Next we show that the assumption of full domain is needed to obtain PAC-bounds by proving that it is impossible to obtain worst-case guarantee PAC-bounds when the domain is not full. We thus classify exactly which utility functions for which PAC bounds are possible.  Finally, we derive lower bounds for all utilities with full domain that demonstrate the tightness of our upper bounds in $S,A,\frac{1}{\varepsilon}, \frac{1}{\delta}$ but with a complicated dependence on $\frac{1}{1-\gamma}$. Finally, for a more restricted class of utilities that include all strongly risk-averse risk measures, we give tighter lower bounds and show that they can outperform the best existing lower bound for $\text{CVaR}_\tau$ in the regime where $\tau$ is very small. 

While we solve the learnability problem, we leave open some gaps in the effective horizon $\frac{1}{1-\gamma}$. Closing these gaps is left for future research and it is not immediately clear how much can be gained from improvements on the upper and lower bounds respectively but we conjecture that improvements can be made on both fronts. Improving the upper bounds however seems to require new analytical techniques as the MB-\OCEVI\ algorithm when $\OCE^u =\E$ is shown to be optimal in the classical setting. Improving the lower bounds might require both new analytical techniques and a new hard-to-learn MDP construction. Finally, we also believe it would be interesting to consider more complex RL settings such as offline RL where data is collected under a behaviour policy which is fixed but unknown or online RL where the data collection process is directly affected by the actions of learning agent. 

\section*{Acknowledgments}
The authors would like to acknowledge the support from  Independent Research Fund Denmark, grant number 1026-00397B. Mohammad Sadegh Talebi was partially supported by Innovation Fund Denmark under Grant 1063-00031B.


\printbibliography

@article{li2000optimal,
  title={Optimal dynamic portfolio selection: Multiperiod mean-variance formulation},
  author={Li, Duan and Ng, Wan-Lung},
  journal={Mathematical finance},
  volume={10},
  number={3},
  pages={387--406},
  year={2000},
  publisher={Wiley Online Library}
}

@article{mortensen2025entropic,
  title={Recursive Entropic Risk Optimization in Discounted MDPs: Sample Complexity Bounds with a Generative Model},
  author={Mortensen, Oliver and Talebi, Mohammad Sadegh},
  journal={arXiv preprint arXiv:2506.00286},
  year={2025}
}

@inproceedings{lee2025risk,
  title={Risk-Averse Constrained Reinforcement Learning with Optimized Certainty Equivalents},
  author={Lee, Jane H and Saglam, Baturay and Pougkakiotis, Spyridon and Karbasi, Amin and Kalogerias, Dionysis},
  booktitle={The Thirty-ninth Annual Conference on Neural Information Processing Systems},
  year={2025}
}

@book{dempster2002risk,
  title={Risk management: value at risk and beyond},
  author={Dempster, Michael Alan Howarth},
  year={2002},
  publisher={Cambridge University Press}
}

@article{lattimore_near-optimal_2014,
  title={Near-optimal {PAC} bounds for discounted {MDP}s},
  author={Lattimore, Tor and Hutter, Marcus},
  journal={Theoretical Computer Science},
  volume={558},
  pages={125--143},
  year={2014},
  publisher={Elsevier}
}

@article{strehl_analysis_2008,
  title={An analysis of model-based interval estimation for {M}arkov decision processes},
  author={Strehl, Alexander L and Littman, Michael L},
  journal={Journal of Computer and System Sciences},
  volume={74},
  number={8},
  pages={1309--1331},
  year={2008},
  publisher={Elsevier}
}

@article{bauerle_markov_2022,
  title={Markov decision processes with recursive risk measures},
  author={B{\"a}uerle, Nicole and Glauner, Alexander},
  journal={European Journal of Operational Research},
  volume={296},
  number={3},
  pages={953--966},
  year={2022},
  publisher={Elsevier}
}

@article{bauerle_markov_2024,
  title={Markov decision processes with risk-sensitive criteria: {A}n overview},
  author={B{\"a}uerle, Nicole and Ja{\'s}kiewicz, Anna},
  journal={Mathematical Methods of Operations Research},
  volume={99},
  number={1},
  pages={141--178},
  year={2024},
  publisher={Springer}
}

@inproceedings{agarwal_model-based_2020,
  title={Model-based reinforcement learning with a generative model is minimax optimal},
  author={Agarwal, Alekh and Kakade, Sham and Yang, Lin F},
  booktitle={Conference on Learning Theory},
  pages={67--83},
  year={2020},
  organization={PMLR}
}

@article{asienkiewicz2017note,
  title={A note on a new class of recursive utilities in {M}arkov decision processes},
  author={Asienkiewicz, Hubert and Ja{\'s}kiewicz, Anna},
  journal={Applicationes Mathematicae},
  volume={44},
  pages={149--161},
  year={2017},
  publisher={Instytut Matematyczny Polskiej Akademii Nauk}
}

@article{bauerle_more_2014,
  title={More risk-sensitive {M}arkov decision processes},
  author={B{\"a}uerle, Nicole and Rieder, Ulrich},
  journal={Mathematics of Operations Research},
  volume={39},
  number={1},
  pages={105--120},
  year={2014},
  publisher={INFORMS}
}

@book{sutton1998reinforcement,
  title={Reinforcement learning: An introduction},
  author={Sutton, Richard S and Barto, Andrew G and others},
  volume={1},
  number={1},
  year={1998},
  publisher={MIT press Cambridge}
}

@inproceedings{hu2023tighter,
  title={A tighter problem-dependent regret bound for risk-sensitive reinforcement learning},
  author={Hu, Xiaoyan and Leung, Ho-fung},
  booktitle={International Conference on Artificial Intelligence and Statistics},
  pages={5411--5437},
  year={2023},
  organization={PMLR}
}

@inproceedings{liang2024regret,
  title={Regret bounds for risk-sensitive reinforcement learning with lipschitz dynamic risk measures},
  author={Liang, Hao and Luo, Zhiquan},
  booktitle={International Conference on Artificial Intelligence and Statistics},
  pages={1774--1782},
  year={2024},
  organization={PMLR}
}

@article{fei2020risk,
  title={Risk-sensitive reinforcement learning: Near-optimal risk-sample tradeoff in regret},
  author={Fei, Yingjie and Yang, Zhuoran and Chen, Yudong and Wang, Zhaoran and Xie, Qiaomin},
  journal={Advances in Neural Information Processing Systems},
  volume={33},
  pages={22384--22395},
  year={2020}
}

@inproceedings{fei2021risk,
  title={Risk-sensitive reinforcement learning with function approximation: A debiasing approach},
  author={Fei, Yingjie and Yang, Zhuoran and Wang, Zhaoran},
  booktitle={International Conference on Machine Learning},
  pages={3198--3207},
  year={2021},
  organization={PMLR}
}

@article{fei2021exponential,
  title={Exponential {B}ellman equation and improved regret bounds for risk-sensitive reinforcement learning},
  author={Fei, Yingjie and Yang, Zhuoran and Chen, Yudong and Wang, Zhaoran},
  journal={Advances in neural information processing systems},
  volume={34},
  pages={20436--20446},
  year={2021}
}

@article{li2020breaking,
  title={Breaking the sample size barrier in model-based reinforcement learning with a generative model},
  author={Li, Gen and Wei, Yuting and Chi, Yuejie and Gu, Yuantao and Chen, Yuxin},
  journal={Advances in neural information processing systems},
  volume={33},
  pages={12861--12872},
  year={2020}
}

@article{li2024breaking,
  title={Breaking the sample size barrier in model-based reinforcement learning with a generative model},
  author={Li, Gen and Wei, Yuting and Chi, Yuejie and Chen, Yuxin},
  journal={Operations Research},
  volume={72},
  number={1},
  pages={203--221},
  year={2024},
  publisher={INFORMS}
}

@article{gheshlaghi2013minimax,
  title={Minimax PAC bounds on the sample complexity of reinforcement learning with a generative model},
  author={Gheshlaghi Azar, Mohammad and Munos, R{\'e}mi and Kappen, Hilbert J},
  journal={Machine learning},
  volume={91},
  pages={325--349},
  year={2013},
  publisher={Springer}
}

@inproceedings{hau2023entropic,
  title={Entropic risk optimization in discounted MDPs},
  author={Hau, Jia Lin and Petrik, Marek and Ghavamzadeh, Mohammad},
  booktitle={International Conference on Artificial Intelligence and Statistics},
  pages={47--76},
  year={2023},
  organization={PMLR}
}

@article{borkar2002risk,
  title={Risk-sensitive optimal control for Markov decision processes with monotone cost},
  author={Borkar, Vivek S and Meyn, Sean P},
  journal={Mathematics of Operations Research},
  volume={27},
  number={1},
  pages={192--209},
  year={2002},
  publisher={INFORMS}
}

@article{borkar2002q,
  title={Q-learning for risk-sensitive control},
  author={Borkar, Vivek S},
  journal={Mathematics of operations research},
  volume={27},
  number={2},
  pages={294--311},
  year={2002},
  publisher={INFORMS}
}

@book{shapiro2021LNs,
  title={Lectures on stochastic programming: {M}odeling and theory},
  author={Shapiro, Alexander and Dentcheva, Darinka and Ruszczynski, Andrzej},
  year={2021},
  publisher={SIAM}
}

@article{ahmadi2012entropic,
  title={Entropic value-at-risk: A new coherent risk measure},
  author={Ahmadi-Javid, Amir},
  journal={Journal of Optimization Theory and Applications},
  volume={155},
  pages={1105--1123},
  year={2012},
  publisher={Springer}
}

@article{rashidinejad2021bridging,
  title={Bridging offline reinforcement learning and imitation learning: {A} tale of pessimism},
  author={Rashidinejad, Paria and Zhu, Banghua and Ma, Cong and Jiao, Jiantao and Russell, Stuart},
  journal={Advances in Neural Information Processing Systems},
  volume={34},
  pages={11702--11716},
  year={2021}
}

@book{kakade2003sample,
  title={On the sample complexity of reinforcement learning},
  author={Kakade, Sham Machandranath},
  year={2003},
  publisher={University of London, University College London (United Kingdom)}
}

@article{jaquette1976utility,
  title={A utility criterion for Markov decision processes},
  author={Jaquette, Stratton C},
  journal={Management Science},
  volume={23},
  number={1},
  pages={43--49},
  year={1976},
  publisher={INFORMS}
}

@inproceedings{ernst2006clinical,
  title={Clinical data based optimal {STI} strategies for {HIV}: {A} reinforcement learning approach},
  author={Ernst, Damien and Stan, Guy-Bart and Goncalves, Jorge and Wehenkel, Louis},
  booktitle={Proceedings of the 45th IEEE Conference on Decision and Control},
  pages={667--672},
  year={2006},
  organization={IEEE}
}

@article{delage2010percentile,
  title={Percentile optimization for {M}arkov decision processes with parameter uncertainty},
  author={Delage, Erick and Mannor, Shie},
  journal={Operations research},
  volume={58},
  number={1},
  pages={203--213},
  year={2010},
  publisher={INFORMS}
}

@article{bielecki1999risk,
  title={Risk-sensitive dynamic asset management},
  author={Bielecki, Tomasz R and Pliska, Stanley R},
  journal={Applied Mathematics and Optimization},
  volume={39},
  pages={337--360},
  year={1999},
  publisher={Springer}
}

@article{scutella2013robust,
  title={Robust portfolio asset allocation and risk measures},
  author={Scutella, Maria Grazia and Recchia, Raffaella},
  journal={Annals of Operations Research},
  volume={204},
  number={1},
  pages={145--169},
  year={2013},
  publisher={Springer}
}

@article{howard1972risk,
  title={Risk-sensitive {M}arkov decision processes},
  author={Howard, Ronald A and Matheson, James E},
  journal={Management science},
  volume={18},
  number={7},
  pages={356--369},
  year={1972},
  publisher={INFORMS}
}

@article{chow2014algorithms,
  title={Algorithms for CVaR optimization in MDPs},
  author={Chow, Yinlam and Ghavamzadeh, Mohammad},
  journal={Advances in neural information processing systems},
  volume={27},
  year={2014}
}

@article{bisi2022risk,
  title={Risk-averse policy optimization via risk-neutral policy optimization},
  author={Bisi, Lorenzo and Santambrogio, Davide and Sandrelli, Federico and Tirinzoni, Andrea and Ziebart, Brian D and Restelli, Marcello},
  journal={Artificial Intelligence},
  volume={311},
  pages={103765},
  year={2022},
  publisher={Elsevier}
}

@article{brown2020bayesian,
  title={Bayesian robust optimization for imitation learning},
  author={Brown, Daniel and Niekum, Scott and Petrik, Marek},
  journal={Advances in Neural Information Processing Systems},
  volume={33},
  pages={2479--2491},
  year={2020}
}

@article{bauerle_2011_markov,
  title={Markov decision processes with average-value-at-risk criteria},
  author={B{\"a}uerle, Nicole and Ott, Jonathan},
  journal={Mathematical Methods of Operations Research},
  volume={74},
  pages={361--379},
  year={2011},
  publisher={Springer}
}

@article{kearns1998phased_Q_learning,
  title={Finite-sample convergence rates for {Q}-learning and indirect algorithms},
  author={Kearns, Michael and Singh, Satinder},
  journal={Advances in neural information processing systems},
  volume={11},
  year={1998}
}

@inproceedings{sidford2018variance,
  title={Variance reduced value iteration and faster algorithms for solving {M}arkov decision processes},
  author={Sidford, Aaron and Wang, Mengdi and Wu, Xian and Ye, Yinyu},
  booktitle={Proceedings of the Twenty-Ninth Annual ACM-SIAM Symposium on Discrete Algorithms},
  pages={770--787},
  year={2018}
}

@article{wang2020randomized,
  title={Randomized linear programming solves the {M}arkov decision problem in nearly linear (sometimes sublinear) time},
  author={Wang, Mengdi},
  journal={Mathematics of Operations Research},
  volume={45},
  number={2},
  pages={517--546},
  year={2020},
  publisher={INFORMS}
}

@article{li2024settling,
  title={Settling the sample complexity of model-based offline reinforcement learning},
  author={Li, Gen and Shi, Laixi and Chen, Yuxin and Chi, Yuejie and Wei, Yuting},
  journal={The Annals of Statistics},
  volume={52},
  number={1},
  pages={233--260},
  year={2024},
  publisher={Institute of Mathematical Statistics}
}

@article{sani2012risk,
  title={Risk-aversion in multi-armed bandits},
  author={Sani, Amir and Lazaric, Alessandro and Munos, R{\'e}mi},
  journal={Advances in neural information processing systems},
  volume={25},
  year={2012}
}

@inproceedings{maillard2013robust,
  title={Robust risk-averse stochastic multi-armed bandits},
  author={Maillard, Odalric-Ambrym},
  booktitle={Algorithmic Learning Theory: 24th International Conference, ALT 2013, Singapore, October 6-9, 2013. Proceedings 24},
  pages={218--233},
  year={2013},
  organization={Springer}
}

@article{khajonchotpanya2021revised,
  title={A revised approach for risk-averse multi-armed bandits under {CVaR} criterion},
  author={Khajonchotpanya, Najakorn and Xue, Yilin and Rujeerapaiboon, Napat},
  journal={Operations Research Letters},
  volume={49},
  number={4},
  pages={465--472},
  year={2021},
  publisher={Elsevier}
}

@article{mihatsch2002risk,
  title={Risk-sensitive reinforcement learning},
  author={Mihatsch, Oliver and Neuneier, Ralph},
  journal={Machine learning},
  volume={49},
  number={2},
  pages={267--290},
  year={2002},
  publisher={Springer}
}

@article{marthe2023beyond,
  title={Beyond average return in {M}arkov decision processes},
  author={Marthe, Alexandre and Garivier, Aur{\'e}lien and Vernade, Claire},
  journal={Advances in Neural Information Processing Systems},
  volume={36},
  pages={56488--56507},
  year={2023}
}

@article{marthe2025efficient,
  title={Efficient Risk-sensitive Planning via Entropic Risk Measures},
  author={Marthe, Alexandre and Bounan, Samuel and Garivier, Aur{\'e}lien and Vernade, Claire},
  journal={arXiv preprint arXiv:2502.20423},
  year={2025}
}

@inproceedings{deng2025near,
  title={Near-Optimal Sample Complexity for Iterated {CVaR} Reinforcement Learning with a Generative Model},
  author={Deng, Zilong and Khan, Simon and Zou, Shaofeng},
	booktitle = {The 28th International Conference on Artificial Intelligence and Statistics},
  year={2025}
}

@inproceedings{du2023provably,
  title={Provably Efficient Risk-Sensitive Reinforcement Learning: Iterated {CVaR} and Worst Path},
  author={Du, Yihan and Wang, Siwei and Huang, Longbo},
  booktitle={The Eleventh International Conference on Learning Representations},
  year={2023}  
}

@inproceedings{chen2024provably,
  title={Provably Efficient Iterated CVaR Reinforcement Learning with Function Approximation and Human Feedback},
  author={Chen, Yu and Du, Yihan and Hu, Pihe and Wang, Siwei and Wu, Desheng and Huang, Longbo},
  booktitle={The Twelfth International Conference on Learning Representations},
  year={2024} 
}

@article{tamar2015policy,
  title={Policy gradient for coherent risk measures},
  author={Tamar, Aviv and Chow, Yinlam and Ghavamzadeh, Mohammad and Mannor, Shie},
  journal={Advances in neural information processing systems},
  volume={28},
  year={2015}
}

@article{la2013actor,
  title={Actor-critic algorithms for risk-sensitive {MDP}s},
  author={La, Prashanth and Ghavamzadeh, Mohammad},
  journal={Advances in neural information processing systems},
  volume={26},
  year={2013}
}

@inproceedings{lam2022risk,
  title={Risk-aware reinforcement learning with coherent risk measures and non-linear function approximation},
  author={Lam, Thanh and Verma, Arun and Low, Bryan Kian Hsiang and Jaillet, Patrick},
  booktitle={The Eleventh International Conference on Learning Representations},
  year={2022}
}

@inproceedings{petrik2012approximate,
  title={An approximate solution method for large risk-averse markov decision processes},
  author={Petrik, Marek and Subramanian, Dharmashankar},
  booktitle={Conference on Uncertainty in Artificial Intelligence},
  year={2012}
}

@article{jin2024truncated,
  title={Truncated variance reduced value iteration},
  author={Jin, Yujia and Karmarkar, Ishani and Sidford, Aaron and Wang, Jiayi},
  journal={Advances in Neural Information Processing Systems},
  volume={37},
  pages={117481--117508},
  year={2024}
}

@article{singh1994upper,
  title={An upper bound on the loss from approximate optimal-value functions},
  author={Singh, Satinder P and Yee, Richard C},
  journal={Machine Learning},
  volume={16},
  number={3},
  pages={227--233},
  year={1994},
  publisher={Springer}
}

@inproceedings{wang2025reductions,
  title={A Reductions Approach to Risk-Sensitive Reinforcement Learning with Optimized Certainty Equivalents},
  author={Wang, Kaiwen and Liang, Dawen and Kallus, Nathan and Sun, Wen},
  booktitle={Forty-second International Conference on Machine Learning},
  year={2025}
}

@inproceedings{xu2023regret,
  title={Regret bounds for Markov decision processes with recursive optimized certainty equivalents},
  author={Xu, Wenhao and Gao, Xuefeng and He, Xuedong},
  booktitle={International Conference on Machine Learning},
  pages={38400--38427},
  year={2023},
  organization={PMLR}
}

@article{ben2007old,
  title={An old-new concept of convex risk measures: The optimized certainty equivalent},
  author={Ben-Tal, Aharon and Teboulle, Marc},
  journal={Mathematical Finance},
  volume={17},
  number={3},
  pages={449--476},
  year={2007},
  publisher={Wiley Online Library}
}

@article{follmer2010convex,
  title={Convex and coherent risk measures},
  author={F{\"o}llmer, Hans and Schied, Alexander},
  journal={Encyclopedia of Quantitative Finance},
  pages={355--363},
  year={2010},
  publisher={John Wiley \& Sons Hoboken}
}

@inproceedings{ni2022risk-evar,
  title={Risk-sensitive reinforcement learning via {Entropic-VaR} optimization},
  author={Ni, Xinyi and Lai, Lifeng},
  booktitle={2022 56th Asilomar Conference on Signals, Systems, and Computers},
  pages={953--959},
  year={2022},
  organization={IEEE}
}

@article{zhao2024ra,
  title={{RA-PbRL}: Provably efficient risk-aware preference-based reinforcement learning},
  author={Zhao, Yujie and Escamill, Jose E and Lu, Weyl and Wang, Huazheng},
  journal={Advances in Neural Information Processing Systems},
  volume={37},
  pages={60835--60871},
  year={2024}
}

@article{moharrami2025policy,
  title={A policy gradient algorithm for the risk-sensitive exponential cost mdp},
  author={Moharrami, Mehrdad and Murthy, Yashaswini and Roy, Arghyadip and Srikant, Rayadurgam},
  journal={Mathematics of operations research},
  volume={50},
  number={1},
  pages={431--458},
  year={2025},
  publisher={Informs}
}

@article{hau2023dynamic,
  title={On dynamic programming decompositions of static risk measures in {M}arkov decision processes},
  author={Hau, Jia Lin and Delage, Erick and Ghavamzadeh, Mohammad and Petrik, Marek},
  journal={Advances in Neural Information Processing Systems},
  volume={36},
  pages={51734--51757},
  year={2023}
}

@article{rigter2023one,
  title={One risk to rule them all: A risk-sensitive perspective on model-based offline reinforcement learning},
  author={Rigter, Marc and Lacerda, Bruno and Hawes, Nick},
  journal={Advances in neural information processing systems},
  volume={36},
  pages={77520--77545},
  year={2023}
}

@article{sood2023deep,
  title={Deep reinforcement learning for optimal portfolio allocation: A comparative study with mean-variance optimization},
  author={Sood, Srijan and Papasotiriou, Konstantinos and Vaiciulis, Matas and Balch, Tucker},
  journal={FinPlan},
  volume={2023},
  number={2023},
  pages={21},
  year={2023}
}

@inproceedings{huang2022achieving,
  title={Achieving mean--variance efficiency by continuous-time reinforcement learning},
  author={Huang, Yilie and Jia, Yanwei and Zhou, Xunyu},
  booktitle={Proceedings of the Third ACM International Conference on AI in Finance},
  pages={377--385},
  year={2022}
}

@inproceedings{kamran2020risk,
  title={Risk-aware high-level decisions for automated driving at occluded intersections with reinforcement learning},
  author={Kamran, Danial and Lopez, Carlos Fernandez and Lauer, Martin and Stiller, Christoph},
  booktitle={2020 IEEE Intelligent Vehicles Symposium (IV)},
  pages={1205--1212},
  year={2020},
  organization={IEEE}
}

@article{gangulyTMLRrisk-seeking,
  title={Risk-Seeking Reinforcement Learning via Multi-Timescale EVaR Optimization},
  author={Ganguly, Deep Kumar and Joseph, Ajin George and Girotra, Sarthak and Sekhar, Sirish},
  journal={Transactions on Machine Learning Research}
}
\newpage
\appendix

\section{Risk Measures}
\label{Section:Appendix:RiskMeasures}

In this section, we briefly introduce risk measures. See, e.g., \cite{follmer2010convex} for a good reference that like us model stochastic outcomes as rewards. We here collect some  precise definitions for the reward setting and list some important examples.   

Let $(\Omega, \mathcal{F},\mathbb{P})$ be a background probability space, and $\mathcal{M}$ a convex cone of random variables defined on the background space. That is, for any $X,Y\in \mathcal{M}$ and $\lambda>0$, it holds that $X+Y\in \mathcal{M}$ and $\lambda X\in \mathcal{M}$. 

\begin{definition}[{Risk measure}]
A functional $\psi:\mathcal{M \rightarrow \R}$ is a \emph{risk measure} if it satisfies the following properties:
    \begin{align*}
        &\psi(0) = 0,\qquad &\text{(Normalization)}
        \\
        &\text{if }X\leq Y \text{ then } \psi(X)\geq \psi(Y), \qquad &\text{(Monotonicity)}
        \\
        & \psi(X+c) = \psi(X)-c, \quad \forall c\in \R. \qquad &\text{(Translation invariance)}
    \end{align*}
    If, in addition, $\psi$ satisfies the properties 
     \begin{align*}
        &\psi(cX) = c\psi(X), \quad \forall c>0,\qquad &\text{(Positive homogeneity)}
        \\
        &\psi(X+Y)\leq \psi(X)+\psi(Y),\qquad &\text{(Sub-additivity)}
    \end{align*}   
    it is called a \emph{coherent risk measure}. A weaker notion is \emph{convex risk measure}, which is one obeying 
         \begin{align*}
        &\psi(\lambda X + (1-\lambda)Y) \leq \lambda \psi(X)+(1-\lambda)\psi(Y), \quad \forall \lambda\in[0,1]\,.\qquad &\text{(Convexity)}
    \end{align*} 
    Finally, a risk-measure $\psi$ is called \emph{law-invariant} if $\psi(X)$ only depends on the distribution of $X$ under $\P$.
 \end{definition}
    We now mention some examples of risk measures.
    \paragraph{Entropic Risk Measure (ERM).}
    The risk measure given by
    \begin{align*}
        \text{ERM}_\beta(X) = \frac{1}{\beta}\log\big(\E[e^{-\beta X}] \big)
    \end{align*}
    is known as the entropic risk measure (ERM) with parameter $\beta \neq0$. Notably, ERM is not coherent (see, e.g., \cite{follmer2010convex}) as it is does not satisfy the positive homogeneity property. Letting $\beta\rightarrow 0$ one recovers the expectation $\E[X]$, and letting $\beta \rightarrow \infty$ yields the essential infimum risk measure. 

    \paragraph{Value-at-Risk (VaR).}~The risk measure given by
    \begin{align*}
        \text{VaR}_\tau(X) := q_\tau(X) := \inf \{ x\in \R:F_X(x)\geq\tau\}
    \end{align*}
    is called the Value-at-Risk (VaR) at level $\tau\in(0,1)$. VaR is in general not sub-additive, and hence also not coherent. 

     \paragraph{Conditional Value-at-Risk (CVaR).}~The risk measure given by 
    \begin{align*}
        \text{CVaR}_\tau(X):= \frac{1}{\tau}\int_0^\tau\text{VaR}_u(X)du
    \end{align*}
    is known as the Conditional Value-at-Risk (CVaR), or sometimes as the expected shortfall (ES). It is known to be a coherent risk-measure. 

    \paragraph{Mean-variance criterion}
    The risk-measure $\text{MV}(X):=\frac{1}{2}\text{Var}(X)-\E[X]$ is called the mean-variance criterion risk measure and is one-way to explicitly using the variance to account for risk.

    \paragraph{Essential infimum.}~The functional $-\text{Essinf}(X)$
    is a risk-measure. It can reasonable be said to be the most risk-averse risk measure as it only takes the worst-case outcome of a random variable into consideration and ignores all other distributional aspects of the random variable. It can be obtained as $\lim_{\beta\rightarrow \infty}\text{ERM}_\beta(X)$ and $\lim_{\tau\rightarrow 0}\text{VaR}_\tau(X)$

\section{Convergence of UVI}
\label{section:ConvergenceOfUVI}


\begin{lemma}
\label{lemma:OCE-contraction}
    The operator $\T$ is a $\gamma$-contraction with respect to $\|\cdot \|_\infty$.
\end{lemma}

\begin{proof}
Consider two maps $Q: \mathbb{R}^{S\times A}\rightarrow \mathbb{R}^{S\times A}$ and $W: \mathbb{R}^{S\times A}\rightarrow \mathbb{R}^{S\times A}$, and let $Q'=\mathcal{T}Q$ and $W'=\mathcal{T}W$ be their respective $\mathcal{T}$-transforms. Let $(s,a)$ be any pair such that $|Q'(s,a)-W'(s,a)| = \|Q'-W'\|_\infty$, and assume without loss of generality that $Q'(s,a) \geq W'(s,a)$. Further, define 
\begin{align*}
    V(s) := \max_a Q(s,a)\,, \qquad  \
    X(s)  := \max_a W(s,a)\, .
\end{align*}
Then by monotonicity and consistency of $\rho$ it holds that
\begin{align*}
    \|Q'-W'\| & = Q'(s,a) - W'(s,a)
    \\
    & = \gamma \Big(\rho_{s,a}(V(s')) - \rho_{s,a}(X(s'))   \Big)
    \\
    & = \gamma  \Big( \rho_{s,a}(X(s')+V(s')-X(s')-\rho_{s,a}(X(s')) \Big)
    \\
    & \leq \gamma \Big(\rho_{s,a}(X(s')+\|V-X\|)-\rho_{s,a}(X(s')) \Big)
    \\
    & = \gamma\|V-X\|
    \\
    & \leq \gamma \|Q-W\|.
\end{align*}
\end{proof}

From this fact we get the following convergence guarantee on \OCEVI\ (Algorithm \ref{Alg:OCEVI}): 

\begin{lemma}
\label{lemma:ConvergenceRateVI}
    If $k\geq \log\big(\frac{1}{2\varepsilon(1-\gamma)}\big)/\log(1/\gamma)$,  then $\|Q^*-Q_k\|\leq \varepsilon$.
\end{lemma}

\begin{proof}
Since $\T$ is a $\gamma$-contraction, we have that 
\begin{align*}
    \|Q_k-Q^*\| & = \|\T Q_{k-1}-\T Q^* \| \leq \gamma\|Q_{k-1}-Q^*\|\,,
\end{align*}
from which it follows that $\|Q_0-Q^*\|\leq \gamma^k\|Q_0-Q^*\|$. Furthermore, by definition  of $Q_0$, it holds that $\|Q_0-Q^*\|\leq \frac{1}{2(1-\gamma)}$. Solving $\frac{\gamma^k}{2(1-\gamma)}<\varepsilon$ for $\gamma$ yields the announced result. 
\end{proof}

\section{Missing Lemmas for Upper Bounds}
\label{app:missing_proofs}

\subsection{Proof of Lemma \ref{lemma:DeteriorationOfGreedyPolicy}}
We first prove a result that bounds the value of a greedy policy by the value-function for which the policy is greedy. 
The result is a generalization of \cite{singh1994upper} to general risk measures. We use the notation $\rho_{s,a}(V(s'))$ as shorthand for $\rho$ applied to the categorical random variable $X$ which takes values in the set $\{V(s')\}_{s'\in \S}$ with probabilities given by $\P(X=V(s')) = P(s'|s,a)$.

\greedyPolicyBound*

\begin{proof}
Let $\bar{s}$ be a state such that $\|V^*-V^G\|=V^*(\bar{s})-V^G(\bar{s})$, where $V^G:=V^{\pi_G}$. We then consider the two actions $a^*:=\pi^*(\bar{s})$ and $a^G:=\pi^G(\bar{s})$; ties can be broken arbitrarily. Since $\pi^G$ is greedy with respect to~$\overline{V}$, we have that 
\begin{align*}
    R(\bar{s},a^*) + \gamma \rho_{\bar{s},a^*}(\overline{V}(s')) \leq  R(\bar{s},a^G) +\gamma \rho_{\bar{s},a^G}(\overline{V}(s'))\,.
\end{align*}
By assumption, it holds for any $s\in \S$ that
\begin{align*}
    V^*(s)-\varepsilon\leq \overline{V}(s)\leq V^*(s)+\varepsilon.
\end{align*}
Since $\rho$ is monotone and translation invariant it follows that
\begin{align*}
R(\bar{s},a^*) + \gamma \rho_{\bar{s},a^*}(\overline{V}(s'))& \geq R(\bar{s},a^*) + \gamma \rho_{\bar{s},a^*}(V^*(s')-\varepsilon) 
\\
& = R(\bar{s},a^*) + \gamma \rho_{\bar{s},a^*}(V^*(s'))-\gamma \varepsilon\,,
\end{align*}
and by a similar argument
\begin{align*}
    R(\bar{s},a^G) + \gamma \rho_{\bar{s},a^G}(\overline{V}(s'))  \leq R(\bar{s},a^G) + \gamma \rho_{\bar{s},a^G}(V^*(s'))+\gamma\varepsilon,
\end{align*}
yielding the inequality
\begin{align*}
   R(\bar{s},a^*)-R(\bar{s},a^G)\leq 2\gamma\varepsilon  + \gamma\big(\rho_{\bar{s},a^G}(V^*(s')) -\rho_{\bar{s},a^*}(V^*(s')\big)\,.
\end{align*}
Combining the previous inequalities we finally see that
\begin{align*}
    V^*(\bar{s})-V^G(\bar{s}) & = R(\bar{s},a^*)-R(\bar{s},a^G) +\gamma \rho_{\bar{s},a^*}(V^*(s'))-\gamma \rho_{\bar{s},a^G}(V^G(s'))
    \\
    & \leq 2\gamma\varepsilon  + \gamma\rho_{\bar{s},a^G}(V^*(s')) -\gamma\rho_{\bar{s},a^*}(V^*(s')  + \gamma \rho_{\bar{s},a^*}(V^*(s'))-\gamma \rho_{\bar{s},a^G}(V^G(s'))
    \\
    &  = 2\gamma \varepsilon +\gamma\big(\rho_{\bar{s},a^G}(V^*(s')-\rho_{\bar{s},a^G}(V^G(s')) \big)
    \\
    & = 2\gamma \varepsilon + \gamma \|V^*-V^G\|,
\end{align*}
from which the result follows. 
\end{proof}

\subsection{Proof of Lemma \ref{lemma:SimulationLemma}}

\OCESimLemma*

In the following proof, we use the following convention. We write $[V|P]$ for the random variable taking values given by the vector $V$ with probability distribution $P$.

\begin{proof}
We suppress $\pi$ from the notation since it is fixed throughout. 

Let $(s,a)$ be a state-action pair such that $\|Q-\widetilde{Q}\| =|Q(s,a)-\widetilde{Q}(s,a)|$ and assume that $Q(s,a)\geq \widetilde{Q}(s,a).$ Note that from monotonicity and consistency of $\rho$ it follows that
\begin{align}
    \rho([\widetilde{V}|\widetilde{P}(\cdot|s,a)]) & = \rho([V+\widetilde{V}-V|\widetilde{P}_{s,a}])
    \\
    & \geq \rho([V-\|V-\widetilde{V}\|_\infty|\widetilde{P}_{s,a}])
    \\
    & = -\|V-\widetilde{V}\| + \rho([V|\widetilde{P}_{s,a}])
\end{align}
which along with the Bellman recursion implies
\begin{align}
   \|Q-\widetilde{Q}\| & = R(s,a)+\gamma\rho([V|P{s,a}(\cdot)])-R(s,a)-\gamma\rho([\widetilde{V}|\widetilde{P}_{s,a}])
   \\
   & \leq \gamma\|V-\widetilde{V}\| + \gamma \big(\rho([V|P_{s,a}(\cdot)])-\rho([V|\widetilde{P}_{s,a}]) \big)
\end{align}
which then implies
\begin{align}
    \|Q-\widetilde{Q}\| \leq \frac{\gamma}{1-\gamma}
\big(\rho([V|P_{s,a}])-\rho([V|\widetilde{P}_{s,a}]) \big)
\end{align}
Using that the OCE is given by the solution to an optimization problem and using Proposition 2.1 in \cite{ben2007old} showing that it suffices to optimize over the interval $\eta\in[0,\effh]$, we then have 
\begin{align}
    \rho([V|P_{s,a}(\cdot)])-\rho([V|\widetilde{P}(\cdot)]) & = \sup_{\eta \in[0,\effh]}\{\eta+\sum_{s'\in S}P_{s,a}(s')[u(V(s')-\eta)] \}
    \\
    & -\sup_{\eta \in[0,\effh]}\{\eta+\sum_{s'\in S}\widetilde{P}_{s,a}(s')[u(V(s')-\eta)]\}
    \\
    & \leq \sup_{\eta \in [0,\effh]}\sum_{s'\in S}[P_{s,a}(s')-\widetilde{P}_{s,a}(s')]u(V(s')-\eta)\
    \\
    & \leq \sup_{\eta \in [0,\effh]}\bigg|\sum_{s'\in S}[P_{s,a}(s')-\widetilde{P}_{s,a}(s')]u(V(s')-\eta)\ \bigg|
\end{align}
thus concluding this case.

The proof for the case $\widetilde{Q}(s,a)< Q(s,a)$ is similar, but instead uses the fact that 
\begin{align*}
    \rho([\widetilde{V}|\widetilde{P}_{s,a}(\cdot)]) & = \rho([\widetilde{V}-V+V|\widetilde{P}_{s,a}(\cdot)]) \leq \|V-\widetilde{V}\| + \rho([V|\widetilde{P}_{s,a}(\cdot)]),
\end{align*}
from which we obtain 
\begin{align}
    \|Q-\widetilde{Q}\| & = \gamma\big(\rho([\widetilde{V}|\widetilde{P}])- \rho([V|P]) \big)
    \\
    & \leq \gamma\|V-\widetilde{V}\| + \gamma\big(\rho([V|\widetilde{P}])-\rho([V|P]) \big).
\end{align}
The proof of this case follows by noting that 
\begin{align*}
    \rho([V|\widetilde{P}])-\rho([V|P]) & \leq \sup_{\eta \in [0,\effh]}\sum_{s'\in S}[\widetilde{P}_{s,a}(s')-{P}_{s,a}(s')]u(V(s')-\eta)\
    \\
    & \sup_{\eta \in [0,\effh]}\bigg|\sum_{s'\in S}[P_{s,a}(s')-\widetilde{P}(s')]u(V(s')-\eta)\ \bigg|.
\end{align*}
\end{proof}

\subsection{Concentration via Hoeffding's Inequality}
The next result is an application of Hoeffding's inequality to establish a bound on the number of samples needed for the expression inside the supremum in Lemma \ref{lemma:SimulationLemma} for a fixed $\eta$ to concentrate. Let $\widehat{P}_{s,a} = \frac{1}{N}\sum_{n=1}^N \mathbbmss 1_{\{X_n=s'\}}$ where $X_n$ is sampled from $\{1,\ldots,S\}$ with probabilities according to $P_{s,a}$.
\begin{lemma}[Hoeffding bound]
\label{lemma:Hoeffding}
Fix $\eta \in [0,\effh]$ and $(s,a)\in \S\times \A$, and let $V\in \R^S$. 
If the number $N$  of samples  from the state-action pair $(s,a)$  satisfies $N\geq \frac{2}{\varepsilon^2}\big[u(-\effh)\big]^2\log(2/\delta)$, then with probability exceeding $1-\delta$,
\begin{align*}
    \bigg|\sum_{s'}[P_{s,a}(s')-\widehat{P}_{s,a}(s')]u(V(s')-\eta)\bigg|\leq \varepsilon\,.
\end{align*}
\end{lemma}
\begin{proof}
        We first observe that for the random variable $\sum_{s'\in \S}\mathbbmss 1_{\{X_n=s'\}}u(V(s')-\eta)$, we have that 
    \begin{align*}
        \E\bigg[\sum_{s'}\mathbbmss 1_{\{X_n=s'\}}u(V(s')-\eta)\bigg] & = \sum_{s'}\E\big[\mathbbmss 1_{\{X_n=s'\}}\big]u(V(s')-\eta)
        \\
        & = \sum_{s'}P(s'|s,a)u(V(s')-\eta)
    \end{align*}
    and that it is bounded in  $[u(-\effh),u(\effh)]\subset [u(-\effh),-u(-\effh)]$.
    Also, since
    \begin{align*}
        \sum_{s'}\widehat{P}_{s,a}(s')u(V(s')-\eta)= \frac{1}{N}\sum_{n=1}^N \sum_{s'}\mathbbmss 1_{\{X_n=s'\}}u(V(s')-\eta),
    \end{align*}
    we have by Hoeffding's inequality that 
    \begin{align*}
        \P\bigg( \bigg|\sum_{s'}[P_{s,a}(s')-\widehat{P}_{s,a}(s')]u(V(s')-\eta)\bigg|\geq \varepsilon \bigg)\leq 2 \exp\bigg(-\frac{2N\varepsilon^2}{\big(-2u(-\effh))^2}\bigg),
    \end{align*}
    with the right-hand side being smaller than $\delta$ if $N \geq \frac{2}{\varepsilon^2}\big[u(-\effh)\big]^2\log(2/\delta)$.
\end{proof}

\subsection{Proof of Lemma \ref{lemma:BoundingSimulationTerm-combined}}


The proof of Lemma \ref{lemma:BoundingSimulationTerm-combined} follows from the following two lemmas (Lemma \ref{lemma:BoundingSimulationTerm} and Lemma \ref{lemma:BoundingDerivative}). 

\begin{restatable}{lemma}{BoundingSimulationTerm}
\label{lemma:BoundingSimulationTerm}
    Let $V\in \R^S$. For $\eta^*,\bar{\eta}\in[0,\effh]$, 
    it holds that 
    \begin{align*}
    \max_{s,a}\bigg|\sum_{s'}(P_{s,a}(s')-&\widehat{P}_{s,a}(s'))u(V(s')-\eta^*)\bigg| \\
    &\leq  \max_{s,a}\bigg|\sum_{s'}(P_{s,a}(s')-\widehat{P}_{s,a}(s'))u(V(s')-\bar{\eta})\bigg|+ u'_+(-\effh)|\eta^*-\bar{\eta}|\,.
    \end{align*}
\end{restatable}

\begin{proof}
    Since $u$ is increasing and concave, we have for $x,y\in[-\effh,\effh]$ where $x\leq y$ that 
    \begin{align*}
        u(y) &\leq u(x)+u'_+(x)(y-x) \leq u'_+(-\effh)(y-x)\,,
    \end{align*}
    and so 
    \begin{align*}
        0\leq u(y)-u(x)\leq u'_+(-\effh)(y-x).
    \end{align*}
    Similarly, if $y\leq x$, it holds that 
    \begin{align*}
        0\leq u(x)-u(y)\leq u'_+(-\effh)(x-y). 
    \end{align*}
    Combining these, the conclusion follows. Plugging in $x=V(s')-\eta^*$ and $y=V(s')-\bar{\eta}$, we thus get 
    \begin{align}
        |u(V(s')-\eta^*)-u(V(s')-\bar{\eta})|\leq u'_+(-\effh)|\eta^*-\bar{\eta}|\,.
    \end{align}
    Next we notice that 
    \begin{align*}
        \sum_{s'}(P_{s,a}(s')-\widehat{P}_{s,a}(s'))u(V(s')-\eta^*) &= \sum_{s'}(P_{s,a}(s')-\widehat{P}_{s,a}(s'))u(V(s')-\bar{\eta}) \\
        &+\sum_{s'}(P_{s,a}(s')-\widehat{P}_{s,a}(s'))\big[u(V(s')-\eta^*)-u(V(s')-\bar{\eta})\big]\,,
    \end{align*}
    and since       
    \begin{align*}
        -u'_+(-\effh)|\eta^*-\bar{\eta}| \leq \sum_{s'}(P_{s,a}(s')-\widehat{P}_{s,a}(s'))\big[u(V(s')-\eta^*)-u(V(s')-\bar{\eta})\big] \leq u'_+(-\effh)|\eta^*-\bar{\eta}|\,,
    \end{align*}
    we obtain that 
    \begin{align*}
        \bigg|\sum_{s'}(P_{s,a}(s')-\widehat{P}_{s,a}(s'))u(V(s')-\eta^*)-\sum_{s'}(P_{s,a}(s')-\widehat{P}_{s,a}(s'))u(V(s')-\bar{\eta})\bigg|\leq u'_+(-\effh)|\eta^*-\bar{\eta}|\,.
    \end{align*}
    Finally, by the reverse triangle inequality we have 
    \begin{align*}
   \bigg| \big|\sum_{s'}&(P_{s,a}(s')-\widehat{P}_{s,a}(s'))u(V(s')-\eta^*)\big|-\big|\sum_{s'}(P_{s,a}(s')-\widehat{P}_{s,a}(s'))u(V(s')-\bar{\eta})\big| \bigg| \\
   &\leq
     \bigg|\sum_{s'}(P_{s,a}(s')-\widehat{P}_{s,a}(s'))u(V(s')-\eta^*)-\sum_{s'}(P_{s,a}(s')-\widehat{P}_{s,a}(s'))u(V(s')-\bar{\eta})\bigg| \\
     & 
    \leq u'_+(-\effh)|\eta^*-\bar{\eta}|
    \end{align*}
    for any $(s,a)\in \S\times \A$. The result follows by taking maximum on both sides. 
\end{proof}

\begin{restatable}{lemma}{BoundingDerivative}
\label{lemma:BoundingDerivative}
    If the set $D$ of discretization points satisfies $|D|\geq \frac{u'_+(-\effh)}{\varepsilon(1-\gamma)}$, it holds that 
    \begin{align*}
        \min_{\bar{\eta}\in D}u'_+(-\effh)|\eta^*-\bar{\eta}|<\frac{\varepsilon}{2}\,.
    \end{align*}
\end{restatable}

\begin{proof}
  For any interval $[0,L]$ of length $L$ that is discretized equidistantly by $|D|$ points, the interval length between any two discretization points is $I=\frac{L}{|D|+1}$ and thus for any point $x\in [0,L]$ its distance to its nearest discretization point is $\frac{I}{2}=\frac{L}{2(|D|+1)}$. To ensure that the distance of any $x\in[0,L]$ to its nearest point is less than $\xi>0$ solving for $D$ shows that it suffices that  $D\geq \frac{L}{2\xi}$. Plugging in $L=\effh$ and $\xi = \frac{\varepsilon}{2u'_+(-\effh)}$ the result follows.
\end{proof}

\section{Proof of Theorem \ref{theorem:ImpossibilityPolicy}}

\ImpossibilityPolicy*

\begin{proof}
We consider a class of MDPs with $3$ states $s_0, s^G, s^B$ and $2$ actions $a_1$ and $a_2$. We will index the MDPs in $\M$ by $\{1,2\}\times (0,1)$ and write $M_i^p$ for $i\in \{1,2\}$ and $p\in(0,1)$. The state $s^G$ is absorbing under any action and yields a reward of 1 under any action. The state $s^B$ is absorbing under any action and yields zero reward under any action. The state $s_0$ yields zero reward under any action and under MDP $M_j^p\in \mathbb{M}$ we have $\P_{j,p}(s^G|s_0,a_j)=p$, $\P_{j,p}(s^B|s_0,a_j)=1-p$ and $\P_{j,p}(s^G|s_0,a^-)=1$ where $a^-$ denotes the other action that is not $a_j$.

By the exact same argument as in Theorem \ref{theorem:ImpossibilityValue} by picking $\gamma$ so large that $\effh>\xi=:-\inf \text{dom}(u)$ we get on any member $M\in \M$ the following lower bound on the value-gap: $V^*(s_0)-V^{a^-}(s_0)>\gamma(\effh-\xi)=:\Delta>0$ between the optimal action $a^*$ and the other action $a^-$. Picking $\varepsilon<\Delta$ only the optimal policies are $\varepsilon$-good. 

The next part of the argument is again a likelihood-ratio type argument that if $p$ is close enough to $1$ the samples needed to tell two MDPs apart will also need to be very large. 
Let $\mathcal{U}$ be an algorithm that picks an action $a$ based $n_1$ tries of action $a_1$ and $n_2$ tries of action $a_2$ where the total number of samples $N = n_1+n_2$. Let $L_i(m_1,m_2,n_1,n_2)$ denote the probability of observing $m_1$ successes by trying $a_1$ for $n_1$ times, and $m_2$ successes by trying $a_2$ for $n_2$ times under hypothesis $H_i$. It is clear to see that 
\begin{align*}
    L_1(m_1,m_2,n_1,n_2)  =p^{m_1}\mathbbmss 1_{[m_2=n_2]}, \qquad     L_2(m_1,m_2,n_1,n_2)  =p^{m_2} \mathbbmss 1_{[m_1=n_1]}
\end{align*}
 We want to evaluate this likelihood-ratio on the event that all tries turn out to be successes, that is on the event 
$
\mathcal{E}:=\{m_1=n_1 \text{ and }m_2=n_2  \}, 
$
where clearly $\P_{1,p}(\mathcal{E})=p^{n_1}$. Assuming $\delta >\frac{1}{2}$, we observe that 
\begin{align*}
    \P_{1,p}(B_2) = \E_1[\mathbbmss 1_{B_2}] = \E_2\bigg[\frac{L_1(n_1,n_2,n_1,n_2)}{L_2(n_1,n_2,n_1,n_2)}\mathbbmss 1_{\mathcal{E}}\mathbbmss 1_{B_2}\bigg]\geq p^{n_1}\P_{2,p}(B_2)>\frac{1}{2}p^{n_1}\geq \frac{1}{2}p^N\,.
\end{align*}
Solving $\frac{1}{2}p^{N}  = \delta$, we find that if $N\leq \frac{\log(1/2\delta)}{\log(1/p)}$, the algorithm that is $(\varepsilon,\delta)$-correct on $M_2^p$ cannot also be $(\varepsilon,\delta)$-correct on $M_1^p$. Finally, by taking the limit $p\rightarrow 1$ the result follows. 
\end{proof}

\section{Lower Bounds}
In this section, we provide two template constructions for lower bounds for value and policy learning and then prove the lower bounds.

\subsection{Hard-to-learn MDP constructions}

In this section we describe the template constructions for the hard-to-learn MDPs for value learning and policy learning. The constructions and proof techniques borrow from \cite{gheshlaghi2013minimax} and \cite{mortensen2025entropic}.

\subsubsection{Value Learning}

For a lower bound we construct the following class of MDPs with $S':=S+2$ states and $A$ actions where the first states are labeled $S_1,,...,s_S,s^G,s^B$ and the actions are labeled $a_1,...,a_A$. The states $s^G$ and $s^B$ are absorbing under any actions and $R(s^G,a)=1$ for all $j$ and $R(s^B,a)=0$ for all $a\in A$. For the states $s\in \{s_1,...,s_S\}$, we have that $R(s,a)=0$ for all $a\in A$. We have $SA$ state-action pair combinations from $\{s_1,...,s_S\}\times A =:Z$ on which we assume some ordering allowing us to write $z_i, i\in[SA]$. Finally for all state-action pairs $z_i\in [SA]$ we have $P(s^G|z_i) = q_i$ and $P(s^B|z_i)=1-q_i$ for some $q_i\in[0,1]$. The structure of this class of MDPs allows us to get lower bounds on the samples needed to learn the $Q$-value of each state-action pair $z_i$ and then use the fact that samples used to learn the $Q$-values for different state-action pairs bring no information on each other to get the final bound. 

In this class of MDPs, we will usually for each $z_i\in Z$ consider two MDPs $M_0$ where $q_i=p$ and $M_1$ where $q_i=p+\alpha$ where $p\in(\frac{1}{2},1)$ and $\alpha\in(0,\frac{1-p}{5})$ and denote the corresponding optimal value functions by $Q^*_0$ and $Q^*_1$. Furthermore we 


\begin{theorem}
\label{theorem:valueLBconstruction}

    Assume $\gamma\geq \frac{1}{2}$, $\delta<\frac{1}{4}$ and $u\in \U_1$. If there exists $p\in(\frac{1}{2},1)$ and $\Delta>0$ such that for any $a\in(0,\frac{1-p}{5})$ and any $z_i\in Z$ it holds that
    \begin{align}
        Q_1^*(z)-Q_0^*(z)\geq \gamma\alpha \Delta
    \end{align}
    then  there exists $\bar{\varepsilon}(u,\effh)$ and constants $c_1,c_2>0$ such that if the total number of samples is less than 
    \begin{align}
        T\leq \frac{1}{c_1}\frac{SA \Delta^2 p(1-p)}{\varepsilon^2}\log(\frac{SA}{c_2\delta})
    \end{align}
    for any algorithm $\mathcal{U}$ and any $\varepsilon\in(0,\bar{\varepsilon})$, there exists some MDP $M_i\in \M$ where $\P(\|Q^*_m-Q^\mathcal{U}\|>\varepsilon)>\delta$.
\end{theorem}

\begin{proof}
By assumption we can pick $p\in(\frac{1}{2},1)$ and $\Delta>0$ such that for $z\in Z$ it holds that 
\begin{align}
    Q_1^*(s,a)-Q_0^*(s,a)\geq \gamma\alpha \Delta \geq \frac{\alpha \Delta}{2}.
\end{align}
Hence, for any $\varepsilon<\frac{\Delta(1-p)}{20}$ we can pick $\alpha =\frac{4\varepsilon}{\Delta}$ to ensure 
\begin{align*}
    Q_1^*(z)-Q^*_0(z)\geq 2\varepsilon
\end{align*}
and so no output $Q^\mathcal{U}$ can be $\varepsilon$-close to both $Q^*_1$ and $Q^*_2$ simultaneously and therefore the two sets $B_1:=\{|Q_1^*-Q^\mathcal{U}|\leq \varepsilon \}$ and $B_0:=\{|Q_0^*-Q^\mathcal{U}|\leq \varepsilon \}$ are disjoint.


Let $t$ be the number of times the algorithm tries $z_i$. Since $\mathcal{U}$ is $(\varepsilon,\delta)$-correct it holds that $\P_0(B_0)\geq 1-\delta\geq \frac{3}{4}$.

Let $k$ be the number of transitions from $z_i$ to $s^G$ in the t trials. We then define $\theta$ by
    \begin{align*}
    \theta := \exp\Big(-\frac{32\alpha^2t}{p(1-p)} \Big)
    \end{align*}
and the event
\begin{align*}
    \mathcal{E} & = \bigg\{pt-k\leq \sqrt{2p(1-p)t\log(\frac{8}{2\theta})} \bigg\},
\end{align*}
for which, we have $\P_0(\mathcal{E})>\frac{3}{4}$ by Lemma 16 in \cite{gheshlaghi2013minimax} and thus $\P_0(B_0\cap\mathcal{E})>\frac{1}{2}$.
Now by Theorem 9 in \cite{mortensen2025entropic}, we get that 
\begin{align*}
    \P_1(B_0) \geq \P_1(B_0\cap \mathcal{E}) = \E_1[\mathbbmss 1_\mathcal{E}\mathbbmss 1_{B_0}] = \E_0\bigg[\frac{L_1}{L_0}\mathbbmss 1_\mathcal{E}\mathbbmss 1_{B_0} \bigg] \geq \frac{\theta}{4}\E_0[\mathbbmss 1_{\mathcal{E}}\mathbbmss 1_{B_0}] = \frac{\theta}{4}\P_0(\mathcal{E}\cap B_0) \geq \frac{\theta}{8}\,.
\end{align*}
Solving for $t$ in $\frac{\theta}{8}>\delta$ we find 
\begin{align*}
    t<\frac{p(1-p)}{32\alpha^2}\log(\frac{1}{8\delta}) = \frac{\Delta^2p(1-p)}{512\varepsilon^2}\log(\frac{1}{8\delta})
\end{align*}

Since also $B_0\subset B_1^c$ we conclude that if the algorithm $\mathcal{U}$ tries the state-action pair $z_i$ less than 
\begin{align*}
    \widetilde{T}(\varepsilon,\delta) := \frac{\Delta^2p(1-p)}{512\varepsilon^2}\log(\frac{1}{8\delta})
\end{align*}
times under the hypothesis $H_0^i$, then $\P_1(B_1^\complement)>\delta$


Let $n:=SA$. If the total number of transition samples is less than $\frac{n}{2}\widetilde{T}(\varepsilon,\delta)$ there has to be at least $n/2$ state-action pairs $z_i$ that has been tried at most $\widetilde{T}(\varepsilon,\delta)$ times which  we might assume are the state-action pairs $\{z_i\}_{i=1}^{n/2}$ without loss of generalirt. 

Let $T_{i}$ be the number of times the algorithm has tried $z_i$ for $i\leq n/2$
Due to the structure of the MDPs in $\mathbb{M}$ it suffices to only consider algorithms that outputs an estimate of $Q^{\mathcal{U}}_{T_i}$ based on samples from $z_i$ since any other samples cannot possibly yield information on $Q^*(z_i)$. 

By defining the events $\Lambda_i := \{|Q_{M_1}^*(z_i)-Q^\mathcal{U}_{T_i}(z_i) |>\varepsilon\}$ we therefore have that $\Lambda_i$ and $\Lambda_j$ are conditionally independent given $T_i$ and $T_j$. We then have
    \begin{align*}
        \mathbb{P}_1(\{\Lambda_i^c\}_{1\leq i \leq n/2} &\cap \{ T_i\leq \widetilde{T}(\varepsilon,\delta) \}_{1\leq i \leq n/2} ) 
        \\
        & = \sum_{t_1=0}^{\widetilde{T}(\varepsilon,\delta)}\dots \sum_{t_{n/2}=0}^{\widetilde{T}(\varepsilon,\delta)} \mathbb{P}_1(\{T_i = t_i\}_{1\leq i\leq n/2}) \mathbb{P}_1(\{\Lambda_i^c\}_{1\leq i \leq n/2} \cap \{ T_i = t_i \}_{1\leq i \leq n/2} )
        \\
        & = \sum_{t_1=0}^{\widetilde{T}(\varepsilon,\delta)}\dots \sum_{t_{n/2}=0}^{\widetilde{T}(\varepsilon,\delta)} \mathbb{P}_1(\{T_i = t_i\}_{1\leq i\leq n/2}) \prod_{1\leq i\leq n/2}\mathbb{P}_1(\Lambda_i^c \cap \{T_i=t_i\})
        \\
        & = \sum_{t_1=0}^{\widetilde{T}(\varepsilon,\delta)}\dots \sum_{t_{n/2}=0}^{\widetilde{T}(\varepsilon,\delta)} \mathbb{P}_1(\{T_i = t_i\}_{1\leq i\leq n/2})  (1-\delta)^{n/2},
    \end{align*}
where we have used the law of total probability from line one to two and from two to three follows from independence. It now follows directly that 
    \begin{align*}
        \mathbb{P}_1(\{\Lambda_i^c\}_{1\leq i \leq n/2} | \{ T_i\leq \widetilde{T}(\varepsilon,\delta) \}_{1\leq i \leq n/2} )  \leq (1-\delta)^{\frac{n}{2}}\,.
    \end{align*}
Therefore, if the total number of transitions $T$ is less than $\frac{n}{2}\widetilde{T}(\varepsilon,\delta)$, then
\begin{align*}
    \mathbb{P}_1(\|Q^*-Q^\mathcal{U}_T\|>\varepsilon) & \geq \mathbb{P}_1\bigg( \bigcup_{z\in S\times A}\Lambda(z)\bigg)
    \\
    & = 1- \mathbb{P}_1\bigg( \bigcap_{1\leq i\leq n/2}\Lambda_i^c\bigg)
    \\
    & \geq 1- \mathbb{P}_1(\{\Lambda_i^c\}_{1\leq i \leq n/2} | \{ T_{z_i}\leq \widetilde{T}(\varepsilon,\delta) \}_{1\leq i \leq n/2} )
    \\
    & \geq 1-(1-\delta)^{n/2} 
    \\
    & \geq \frac{\delta n}{4}\,.
\end{align*}
By setting $\delta' = \delta \frac{n}{4}$ and substituting back $S'$ and using $\frac{S}{2}\leq S-2$ for $S\geq 2$
\begin{align}
    T = \frac{SA\Delta^2p(1-p)}{2048 \varepsilon^2}\log\Big(\frac{SA}{64\delta}\Big)
    \end{align}
    on the MDP corresponding to the hypothesis $H_0:\{H_0^i|1\leq i\leq n\}$
    it finally holds that $\P_1(\|Q^*_{M_1}-Q^\mathcal{U}_T\|>\varepsilon)>\delta'$.
\end{proof}

\subsubsection{Policy Learning}

The class of MDPs we consider has $S+2$ states labelled $s_1,\ldots,s_S,s^G, s^B$ and $A+1$ actions labelled $a_0,a_1,\ldots,a_A$. The state $s^G$ is absorbing and yields a reward of $R=1$ under all actions. The state $s^B$ is also absorbing and yields a reward of $R=0$ under all actions. All other states $s_1,\ldots,s_S$ yields zero reward under all actions and can only tansition to either $s^G$ or $s^B$ with probabilities depending on the action and the MDP. 

For each state $s_i\in \{s_1,\ldots,s_S\}$ we consider the $a+1$ hypotheses $H^i_0$ and $H^i_l$ for $l\neq 0$ defined as
\begin{align*}
    H^i_0:& q(s_i,a_0) = p+\alpha &&q(s_i,a) = p \text{  for }a\neq a_0
    \\
    H^i_l:& q(s_i,a_0) = p+\alpha &&q(s_i,a) = p \text{  for }a\notin \{a_0,l\} \quad & q(s_i,a_l) = p+2\alpha\,,
\end{align*}
where $p\in (\frac{1}{2},1)$ and $\alpha\in (0,\frac{1-p}{10})$ and $q(s_i,a) :=\P(s^G|s_i,a)$. 

We use $V_l^j(s_i)$ to denote the value function of state $s_i$ in any of the MDPs where $a_l$ is the optimal action in state $s_i$ under any policy for which $\pi(s)=a_j$. From the construction it is clear that this value function in $s_i$ does not depend on the entire policy but only the action taken in $s_i$. 

Note that under $H_l^i$ the optimal action is $a_l$ with the second best option being $a_0$ and all remaining actions being even worse in the sense that $V^*_l(s_i)\geq V^0_l(s_i)\geq V^j_l(s_i)$ where $V^*$ is the optimal value-function and $V^l$ is the value-function under any policy where action $a_l$ is taken in state $s_i$. 


\begin{theorem}
\label{theorem:policyLBconstruction}

    Let $\gamma\geq \frac{1}{2}$, $\delta<\frac{1}{4}$ and $u\in \U_1$ be given. Furthermore, assume that there exist $p\in(\frac{1}{2},1)$ and $\Delta>0$ such for every $\alpha \in(0,\frac{1-p}{10})$, every $l\in\{0,1,\ldots,A\}$, and $s_i \in \{s_1,\ldots,s_S\}$, it holds that 
    \begin{align*}
    V_0^*(s_i)-V_0^l(s_i) \geq \gamma \alpha \Delta, \qquad 
        V^*_l(s_i)-V^0_l(s_i) &  \geq \gamma \alpha \Delta.
    \end{align*}
    Then there exist constants $c_1,c_2>0$ such that if the total number of samples is less than 
    \begin{align*}
        T\leq \frac{1}{c_1}\frac{SA \Delta^2 p(1-p)}{\varepsilon^2}\log(\frac{S}{c_2\delta}),
    \end{align*}
    then     for any algorithm $\mathcal{U}$ and any $\varepsilon\in(0,\frac{\Delta(1-p)}{10})$, there exists some MDP $M_i\in \M$ such that $\P(\|V^*_m-V^{\pi_\mathcal{U}}\|>\varepsilon)>\delta$.    
\end{theorem}


\begin{proof}

 If we choose $\alpha = \frac{2\varepsilon}{\Delta}$ for any $\varepsilon<\frac{\Delta(1-p)}{20}$ any suboptimal action is $\varepsilon$-bad. 



Now that all non-optimal actions are $\varepsilon$-bad, we wish to show that any algorithm that is $(\varepsilon,\delta)$-correct on $H_0^i$, i.e.~choosing the action $a_0$ with probability at least $1-\delta$, will also have a probability of choosing $a_0$ on $H_l^i$ that is larger than $\delta$ provided that $a_l$ is not tried sufficiently many times under $H_0^i$.

Let $\P_l$ and $\E_l$ denote the probability operator and expectation operator under the hypothesis $H^ i_l$. Let $t:=t^i_l$ be the number of times the algorithm tries action $l$ in $s_i$ under $H_0$. Assuming that $\delta\in (0,\frac{1}{4})$ and using that the algorithm is $(\varepsilon,\delta)$-correct it follows that $\P_0(B)\geq 1-\delta \geq \frac{3}{4}$ where $B = \{\pi^\mathcal{U}(s_i) = a_0\}$ is the event that the algorithm chooses the action $a_0$. 

Let $\theta = \exp\big(-\frac{32\alpha^2t}{p(1-p)} \big)$. Fix some $t\in \mathbb{N}$ and let $k$ be the number of transitions to $s^G$ in $t$ trials.

Finally, we define the event $\mathcal{E}$ as 
\begin{align}
    \mathcal{E}  = \bigg\{pt-k\leq \sqrt{2p(1-p)\log(\frac{8}{2\theta})}\bigg\}\,.
\end{align}
From the Chernoff-Hoeffding bound and as shown in \cite{gheshlaghi2013minimax}, we have that $\P_0(\mathcal{E)}>\frac{3}{4}$, and thus, $\P_0(B\cap\mathcal{E)}>\frac{1}{2}$. From Theorem 9 in \cite{mortensen2025entropic}, we get that 
\begin{align}
    \P_1(B) \geq \P_1(B\cap\mathcal{E)} = \E_1[\mathbbmss 1_B \mathbbmss 1_\mathcal{E}] \geq \E_0\bigg[\frac{L_1(W)}{L_0(W)}\mathbbmss 1_\mathcal{E}\mathbbmss 1_B\bigg] \geq \E_0\bigg[\frac{\theta}{4}\mathbbmss 1_{\mathcal{E}}\mathbbmss 1_B\bigg] = \frac{\theta}{4}\P_0(\mathcal{E}\cap B) \geq \frac{\theta}{8}\,.
\end{align}
Now solving for $\frac{\theta}{8}>\delta$, we see that if
\begin{align}
    t<\widetilde{T}(\varepsilon,\delta):=\frac{\Delta^2p(1-p)}{128\varepsilon^2}\log(\frac{1}{8\delta})
\end{align}
then $\P_1(B)>\delta$ and the event $B$ is containing the event that the algorithm does not choose the optimal action $a_l$. 

Since this holds for all the $A$ hypotheses $H_l^i, l=1,2,\ldots, A$, it follows that the algorithm needs at least $\widetilde{T}(\varepsilon,\delta):=A\widetilde{T}(\varepsilon,\delta)$ samples to be $(\varepsilon,\delta)$-correct on the state $s_i$.


Next we use the fact that the structure of the MDPs is such that the information used to determine $\pi^*(s_i)$ carries no information to determine $\pi^*(s_j)$ for $i\neq j$.

If the total number of transition samples is less than $\frac{S}{2}\widetilde{T}(\varepsilon,\delta)$, there has to be at least $\frac{S}{2}$ states in the set $\{s_i\}_{i=1}^S$ for which at least one action (apart from $a_0$) has been tried at most $\widetilde{T}(\varepsilon,\delta)$ times. We might without loss of generality assume that these are the states $\{s_i\}_{i=1}^{S/2}$ and that it is action $a_1$ that has been tried out at most $\widetilde{T}(\varepsilon,\delta)$ times in each of these states. 

Let $T_{i}$ be the number of times the algorithm has tried sampled any action on $s_i$ for $i\leq S/2$.
By the structure of the MDPs in $\mathbb{M}$ it is suffices to only consider algorithms that outputs an estimate of $\pi^{\mathcal{U}}_{T_i}$ based on samples from $s_i$ since any other samples yields no information on $\pi^*(s_i)$.

Let us define the events $\Lambda_i := \{|V_{M_1}^*(s_i)-V^{\pi^\mathcal{U}_{T_i}}(s_i) |>\varepsilon\}$ for $i=1,\ldots, S$. Then, we have that $\Lambda_i$ and $\Lambda_j$ are conditionally independent given $T_i$ and $T_j$. We then have that for the MDP $M_1\in \mathbb{M}$ --the one corresponding to the hypothesis $H_1:=\{H_1^i|1\leq i\leq n\}$-- it holds that
    \begin{align*}
        \mathbb{P}\big(\{\Lambda_i^c\}_{1\leq i \leq S/2} &\cap \{ T_i\leq \widetilde{T}(\varepsilon,\delta) \}_{1\leq i \leq S/2} \big) \\
        & = \sum_{t_1=0}^{\widetilde{T}(\varepsilon,\delta)}\dots \sum_{t_{S/2}=0}^{\widetilde{T}(\varepsilon,\delta)} \mathbb{P}\big(\{T_i = t_i\}_{1\leq i\leq S/2}\big) \mathbb{P}\big(\{\Lambda_i^c\}_{1\leq i \leq S/2} \cap \{ T_i = t_i \}_{1\leq i \leq S/2} \big)
        \\
        & = \sum_{t_1=0}^{\widetilde{T}(\varepsilon,\delta)}\dots \sum_{t_{S/2}=0}^{\widetilde{T}(\varepsilon,\delta)} \mathbb{P}\big(\{T_i = t_i\}_{1\leq i\leq S/2}\big) \prod_{1\leq i\leq S/2}\mathbb{P}\big(\Lambda_i^c \cap \{T_i=t_i\}\big)
        \\
        & = \sum_{t_1=0}^{\widetilde{T}(\varepsilon,\delta)}\dots \sum_{t_{S/2}=0}^{\widetilde{T}(\varepsilon,\delta)} \mathbb{P}\big(\{T_i = t_i\}_{1\leq i\leq S/2}\big)  (1-\delta)^{S/2}\,,
    \end{align*}
where the first line follows from the  law of total probability, and the second line from independence. We now have directly that 
    \begin{align*}
        \mathbb{P}\Big(\{\Lambda_i^c\}_{1\leq i \leq S/2} \Big| \{ T_i\leq \widetilde{T}(\varepsilon,\delta) \}_{1\leq i \leq S/2}\Big)  \leq (1-\delta)^{\frac{S}{2}}\,.
    \end{align*}
Thus, if the total number of transitions $T$ is less than $\frac{S} {2}\widetilde{T}(\varepsilon,\delta)$ on the MDP $M_0$ corresponding to the hypothesis $H_0:\{H_0^i|1\leq i\leq n\}$, then on $M_1$ it holds that
\begin{align*}
    \mathbb{P}(\|V^*-V^{\pi^\mathcal{U}_T}\|>\varepsilon) & \geq \mathbb{P}\bigg( \bigcup_{1\leq i\leq S/2}\Lambda(z)\bigg)
    \\
    & = 1- \mathbb{P}\bigg( \bigcap_{1\leq i\leq S/2}\Lambda_i^c\bigg)
    \\
    & \geq 1- \mathbb{P}\Big(\{\Lambda_i^c\}_{1\leq i \leq S/2} \, \Big| \, \{ T_{z_i}\leq \widetilde{T}(\varepsilon,\delta) \}_{1\leq i \leq S/2} \Big)
    \\
    & \geq 1-(1-\delta)^{S/2} 
    \\
    & \geq \frac{\delta S}{4},
\end{align*}
when $\frac{\delta S}{2}\leq 1$. By setting $\delta' = \delta \frac{S}{4}$ and substituting back $S'$ and $A'$, assuming that $S\geq 4, A\geq 2$ we conclude that if the number of samples is smaller than 
\begin{align*}
    T = \frac{SA\Delta^2p(1-p)}{1024}\log(\frac{S}{64\delta})
    \end{align*}
    on $M_0$, then on $M_1$ it holds that $\P(\|V^*-V^{\pi^\mathcal{U}_T}\|>\varepsilon)>\delta$.
\end{proof}

\subsection{Proofs of Lower Bounds}

In this section, we give the proofs of the lower bounds using the constructions described above. 

\subsubsection{Value Learning Lower Bounds}

\LowerboundGenericValue*

\begin{proof}
On the small MDP type sketched in figure \ref{fig:Lowerbound_maintext} we will give a lower bound on the optimal Q-functions on $z_i$ for two different parameter choices of $q$, namely $q_0=p$ and $q_1=p+\alpha$.

Clearly for any $Q$ we have that 
\begin{align*}
    Q^*(z)=\gamma \max_{\eta\in[0,\effh]}\{\eta+qu(\effh-\eta)+(1-q)u(-\eta)\}
\end{align*}
We use $\eta_1$ and $\eta_0$ to denote the respective maximizers for $q_1$ and $q_0$ respectively.
We then have that 
\begin{align*}
    Q^*_1(z)&-Q^*_0(z) \\
    & = \gamma\big(\eta_1+(p+\alpha)u(\effh-\eta_1+(1-p-\alpha)u(-\eta_1)-\eta_0-pu(\effh-\eta_0)-(1-p)u(-\eta_0) \big)
    \\
    & \geq \gamma \alpha \big(u(\effh-\eta_0)-u(-\eta_0) \big)
\end{align*}
with $\Delta = \big(u(\effh-\eta_0)-u(-\eta_0) \big)$
By Theorem \ref{Theorem:AuxiliaryLowerBoundValue}, there exists $p\in(\frac{1}{2},1)$ such that $\Delta:=u({\effh-\eta_0})-u(\eta_0)>0$ and so the result follows by Theorem \ref{theorem:valueLBconstruction} with $\Phi(u,\effh)=\Delta^2p(1-p)$.
\end{proof}

\LowerboundSpecificValue*

\begin{proof}
    The idea of the proof is similar to that of Theorem \ref{theorem:genericValue} with the same MDP construction. Recall that 
    \begin{align*}
        Q_1^*(z)-Q^*_0(z)\geq \gamma \alpha \big(u\big(\effh-\eta_0\big)-u(-\eta_0) \big).
    \end{align*}
By the assumption on $u$, we have by Theorem \ref{Theorem:AuxiliaryLowerBoundValueRestricted}, we can pick 
$$
p=\frac{1}{2}+\frac{1}{2}\frac{1-u'_+(0)}{u'_+(-\effh)-u'_+(0)}>\frac{1}{2}
$$ 
to ensure that $\eta_0=\effh$ and so that $\Delta = u(\effh-\eta_0)-u(-\eta_0)=-u(-\effh)$. The result then follows from Theorem \ref{theorem:valueLBconstruction}.
\end{proof}

\subsubsection{Policy Learning Lower Bounds}

\LowerboundGenericPolicy*

\begin{proof}
We have for any $s$ and $l\neq 0$
\begin{align*}
    V_0^*(s)&-V_0^l(s) \\
    & = \gamma\big(\eta_1+(p+\alpha)u(\effh-\eta_1)+(1-p-\alpha)u(-\eta_1) - \eta_0-pu(\effh-\eta_0)-(1-p)u(-\eta_0)\big)\\
     &\geq \gamma \alpha \big(u(\effh-\eta_0)-u(-\eta_0) \big),
\end{align*}
where $\eta_1$ is an optimizer of $\max_{\eta\in[0,\effh]}\{\eta+(p+\alpha)u(\effh-\eta)-(1-p-\alpha))u(-\eta)\}$ and $\eta_0$ is an optimizer of $\max_{\eta\in[0,\effh]}\{\eta+pu(\effh-\eta)-(1-p)u(-\eta)\}$. Similarly for all $l=1\ldots,A$,  
\begin{align*}
    V_l^*(s)-V_l^0(s) \geq \gamma \alpha \big(u(\effh-\eta_1)-u(-\eta_1) \big).
\end{align*}
By Theorem \ref{Theorem:AuxiliaryLowerBoundValueRestricted}, there exists $\bar{p}_0<1$ so that for all $p>\bar{p}_0$ it holds that $u(\effh-\eta_0)-u(-\eta_0) >0$ and $\bar{p}_1<1$ so that for $p+\alpha>\bar{p}_1$ it holds that $u(\effh-\eta_1)-u(-\eta_1)>0$. By picking $p$ large enough and $\varepsilon$ sufficiently small then $\alpha$ is also sufficiently small so that $p+2\alpha<1$ and both $u(\effh-\eta_0)-u(-\eta_0) >0$ and $u(\effh-\eta_1)-u(-\eta_1)>0$. Plugging in 
$$\Delta = \min \{ u(\effh-\eta_1)-u(-\eta_1), u(\effh-\eta_0)-u(-\eta_0)   \},
$$ 
the results follows from Theorem \ref{theorem:policyLBconstruction} with $\Phi(u,\effh)=\Delta^2 p(1-p)$.
\end{proof}

\LowerboundSpecificPolicy*

\begin{proof}
    The idea of the proof is similar to that of Theorem \ref{theorem:genericPolicy}. Recall that 
    \begin{align*}
                V_0^*(s)-V_0^l(s) & \geq \gamma \alpha \big(u(\effh-\eta_0)-u(-\eta_0) \big), \qquad          
            V_l^*(s)-V_l^0(s) \geq \gamma \alpha \big(u(\effh-\eta_1)-u(-\eta_1) \big),
    \end{align*}
    and by Theorem \ref{Theorem:AuxiliaryLowerBoundValueRestricted} if both $p$ and $p+\alpha\geq \max\{\frac{1}{2}, 1-\frac{1-u'_+(0)}{u'_+(-\effh)-u'_+(0)}\}$ then the above right hand sides are both equal to $\Delta = |u(-\effh)|$. For sufficiently small $\varepsilon$ we get that $\alpha$ is sufficiently small so that $p=\frac{1}{2}+\frac{1}{2}\frac{1-u'_+(0)}{u'_+(-\effh)-u_+'(0)}$ ensures this. Plugging in for $p$ and $\Delta$ the result follows.
\end{proof}

\section{Auxiliary Results}

\begin{proposition}
\label{proposition:expectationutilities}
    Let $u\in\U_0$ be any utility function for which $u(t)=t$ for all $t\geq 0$. Then $\OCE_u(X)=\E[X]$.
\end{proposition}
\begin{proof}
    Since any OCE satisfies that $\OCE(X)\leq \E[X]$ we have to show that $\OCE(X)\geq \E[X]$.
    By picking $\eta = \text{Essinf}(X)$, we get that $u(X-\eta)=X-\eta$ and so 
    \begin{align}
        \eta+\E[u(X-\eta)] = \eta+\E[X-\eta] = \E[X]\,.
    \end{align}
Thus, by taking the supremum over all $\eta$, the result follows. 
\end{proof}

\begin{proposition}
\label{proposition:DerivativeIsDominatedByU}
    For piecewise differentiable $u\in \U_1$, it holds that $-u(-\effh)\geq u_+'(-\frac{\gamma}{1-\gamma})$.
\end{proposition}
\begin{proof}
Assume that $u$ is differentiable. Then since $u$ is increasing and concave, we have for $t\geq1$ that 
\begin{align*}
    u(-t) & = u(0)+\int_0^{-t}u'(x)dx
     =-\int_{-t}^0 u'(x)dx 
     \leq -\int_{-t}^{-t+1}u'(x)dx
     \leq -u'(-t+1)\,.
\end{align*}
Multiplying by $-1$ on both sides and plugging in $t = \effh$, we get that $-u(-\effh)\geq u'(-\frac{\gamma}{1-\gamma})$. 
The result then follows from partitioning the integral over the different subdomains on which $u$ is differentiable.
\end{proof}

\begin{theorem}
\label{Theorem:AuxiliaryLowerBoundValueRestricted}
    Let $x>0$ and $X$ be a random variable with $\P(X=x)=p$ and $\P(x=0)=1-p$. Let $u\in \U_1^<$ be a strongly risk-averse utility function. If
    \begin{align*}
        p \geq  1-\frac{1-u'_+(0)}{u'_+(-x)-u'_+(0)},
    \end{align*}
    then $\eta^* = x$ is a solution to 
    \begin{align*}
        \sup_{\eta \in [0,x]}\{\eta+pu(x-\eta)+(1-p)u(-\eta)\}.
    \end{align*}
    \end{theorem}
\begin{proof}
    By definition $\eta^*=x$ is a solution if and only if for all $\eta\in [0,x]$, it holds that 
    \begin{align*}
        \eta+pu(x-\eta)+(1-p)u(-\eta) & \leq x+(1-p)u(-x),
    \end{align*}
    where trivially the inequality holds for $\eta = x$, and where for $\eta<x$ the inequality is equivalent to 
    \begin{align*}
        1-p\leq \frac{x-\eta-u(x-\eta)}{-u(-x)+u(-\eta)-u(x-\eta)} = \frac{1-\frac{u(x-\eta)}{x-\eta}}{\frac{-u(-x)+u(-\eta)}{x-\eta}-\frac{u(x-\eta)}{x-\eta}}.
    \end{align*}
    Thus, we wish to find a lower bound on the right-hand side that holds for all $\eta \in [0,x)$. Note that 
    \begin{align*}
        \frac{-u(-x)+u(-\eta)}{x-\eta} & = \frac{u(-x)-u(-\eta)}{-x-(-\eta)} \leq u'_+(-x),
    \end{align*}
    and since for any $b>1$, the map $t\mapsto \frac{1-t}{b-t} $ is decreasing on $t\in[0,1]$ and $\frac{u(x-\eta)}{x-\eta}>u'_+(0)$, by concavity of $u$ we then have
    \begin{align*}
       \frac{1-\frac{u(x-\eta)}{x-\eta}}{\frac{-u(-x)+u(-\eta)}{x-\eta}-\frac{u(x-\eta)}{x-\eta}} & \geq  \frac{1-\frac{u(x-\eta)}{x-\eta}}{u'_+(-x)-\frac{u(x-\eta)}{x-\eta}}
       \\
       & \geq \frac{1-u'_+(0)}{u'_+(-x)-u'_+(0)}, 
    \end{align*}
thus proving the lemma. 
\end{proof}

\begin{theorem}
\label{Theorem:AuxiliaryLowerBoundValue}
    Let $x>0$ and $u\in \U_1$ be given. Then there exists some $\bar{p}\in (\frac{1}{2},1)$ such that for all $p\in(\bar{p},1)$,
    \begin{align*}
        u(x-\eta^*)-u(-\eta^*)>0,
    \end{align*}
    where $\eta^*$ is the solution to 
    \begin{align*}
        \max_{\eta \in[0,x]}\{\eta+pu(x-\eta)+(1-p)u(-\eta)\}.
    \end{align*}
\end{theorem}

\begin{proof}
    We first note that $u(x-\eta)\geq 0$ and $-u(-\eta)\geq 0$ for all $\eta\in[0,x]$ and $u(x-\eta)-u(-\eta)=0$ if and only if both $\eta=0$ and $u(t)=0$ for all $t\geq 0$. 

    If $u(t)$ is not identically zero on $t\geq 0$ then $u(t)>0$ for all $t>0$ and so $u(x-\eta)=0$ implies that $\eta =x$ but since $-u(-x)\geq x>0$ we can in this case conclude $u(x-\eta^*)-u(-\eta^*)>0$ for all $p\in(0,1).$

    Now we treat the case where $u(t)=0$ for all $t\geq 0$ which we partition into two cases: $u(t)=t$ for $t\in[-x,0]$ (Case (i)) and $u(-x)<-x$ (Case (ii)). 

    \textbf{Case (i).} We have that $\eta+pu(x-\eta)+(1-p)u(-\eta) = p\eta$ which is clearly maximized by $\eta^* = x$ and so $u(x-\eta^*)-u(-\eta^*)=-u(-x)>0$ for all $p\in(0,1)$.

    \textbf{Case (ii).} We note that $\eta^*=0$ is a solution if and only if for all $\eta \in [0,x]$ it holds that $\eta+(1-p)u(-\eta)\leq 0$ or equivalently $p\leq 1-\frac{\eta}{-u(-\eta)}$ but since $u(-x)<-x$ the inequality is violated if $p>1+\frac{-x}{-u(-x)}$ and since $\frac{-x}{-u(-x)}>-1$, we can pick
    \begin{align*}
         \bar{p} = \frac{1+\frac{x}{-u(-x)}}{2} \in(\frac{1}{2},1),
    \end{align*}
    ensuring that $\eta =0$ is not a solution (as $\eta = x$ violates the inequality). Thus, $u(x-\eta^*)-u(-\eta^*)>0$.
\end{proof}
\end{document}